\documentclass[accepted]{uai2025} % after acceptance, for a revised version; 
\usepackage[british]{babel}

\usepackage{natbib} % has a nice set of citation styles and commands
\usepackage{mathtools} % amsmath with fixes and additions
\usepackage{booktabs} % commands to create good-looking tables
\usepackage{tikz} % nice language for creating drawings and diagrams

\usepackage{bm}
\usepackage{algorithm}
\usepackage{algpseudocode}
\usepackage{placeins}

\usepackage{makecell}
\usepackage{amssymb}
\usepackage{nccmath}
\usepackage{graphicx}
\usepackage{xcolor}
\usepackage{hyperref}
\usepackage{multirow} 
\usepackage{enumitem}

\newcommand{\X}{\mathbf{X}}
\newcommand{\W}{\mathbf{W}}
\newcommand{\Z}{\mathbf{Z}}

\newcommand{\A}{\mathbf{A}}
\newcommand{\C}{\mathbf{C}}
\newcommand{\Y}{\mathbf{Y}}
\renewcommand{\P}{\mathbf{P}}
\newcommand{\U}{\mathbf{U}}
\newcommand{\V}{\mathbf{V}}

\newcommand{\+}[1]{\mathbf{#1}}
\newcommand{\E}{\mathbb{E}}

\renewcommand{\O}{\textnormal{O}}
\renewcommand{\t}{\textnormal}

\newcommand{\Uu}[1]{\m{U}_{\m{#1}}}
\newcommand{\Vv}[1]{\m{V}_{\m{#1}}}
\newcommand{\Sig}[1]{\mbs{\Sigma}_{\m{#1}}}

\newcommand{\m}[1]{\mathbf{#1}}
\newcommand{\mr}[1]{\mathrm{#1}}
\newcommand{\mc}[1]{\mathcal{#1}}

\newcommand{\mb}[1]{\mathbb{#1}}
\newcommand{\mbs}[1]{\boldsymbol{#1}}

\DeclareMathOperator*{\argmin}{arg\,min}

\newcommand{\OP}{\mr{O}_{\mb{P}}}
\newcommand{\ThetaP}{\Theta_{\mb{P}}}

\newcommand{\mand}{\hspace{0.45em}\mr{and}\hspace{0.45em}}
\usepackage{subcaption}
\usepackage{array}
\newcolumntype{?}{!{\vrule width 1.1pt}}
\usepackage{amsthm}

\usepackage{lipsum}% http://ctan.org/pkg/lipsum
\usepackage{graphicx}% http://ctan.org/pkg/graphicx
\newtheorem{assumption}{Assumption}
\newtheorem{definition}{Definition}
\newtheorem{theorem}{Theorem}
\newtheorem{proposition}{Proposition}
\newtheorem{corollary}{Corollary}
\newtheorem{lemma}{Lemma}
\usepackage{amsmath}

\usepackage{amsfonts}
\usepackage{booktabs}

\title{Unsupervised Attributed Dynamic Network Embedding with Stability Guarantees}

\author[1]{Emma Ceccherini}{}
\author[2]{Ian Gallagher}{}
\author[3]{Andrew Jones}{}
\author[1]{Daniel Lawson}{}
\affil[1]{%
     University of Bristol, U.K.
}
\affil[2]{%
    The University of Melbourne, Australia
}
\affil[3]{%
    University of Edinburgh, U.K.
  }

\begin{document}

\maketitle
\begin{abstract}

Stability for dynamic network embeddings ensures that nodes behaving the same at different times receive the same embedding, allowing comparison of nodes in the network across time. We present attributed unfolded adjacency spectral embedding (AUASE), a stable unsupervised representation learning framework for dynamic networks in which nodes are attributed with time-varying covariate information. To establish stability, we prove uniform convergence to an associated latent position model. We quantify the benefits of our dynamic embedding by comparing with state-of-the-art network representation learning methods on four real attributed networks. To the best of our knowledge, AUASE is the only attributed dynamic embedding that satisfies stability guarantees without the need for ground truth labels, which we demonstrate provides significant improvements for link prediction and node classification.

\end{abstract}

\section{Introduction}

Representation learning on dynamic networks \citep{goyal2020dyngraph2vec,zuo2018embedding,trivedi2019dyrep,zhou2018dynamic,ma2020streaming,hamilton2017inductive, goyal2018dyngem}, which learns a low dimensional representation for each node, is a widely explored problem. 
While most existing network embedding techniques focus solely on the network features, nodes in real-world networks are associated with a rich set of attributes. For example, in a social network, the user's posts are significantly correlated with trust and following relationships, and it has been shown that jointly exploiting both information sources improves learning performance \citep{tang2013exploiting}. 

Network embeddings for static attributed networks include frameworks based on matrix factorisation \citep{yang2015network}, or deep learning \citep{gao2018deep,tu2017cane,tan2023collaborative, sun2016general,zhang2018anrl,li2021deep}. Some existing dynamic network embeddings leverage node attributes, but their exploitation of node attributes is rather limited, as they are usually solely used to initialise the first layer \citep{sankar2020dysat,dwivedi2023benchmarking,liu2021motif,xu2020embedding,xu2020inductive}. 

Approaches that purposefully exploit node attributes include frameworks based on matrix factorisation \citep{liu2020dynamic,li2017attributed}, deep learning  \citep{tang2022dynamic,ahmed2024learning, wei2019lifelong}, or Bayesian modelling \citep{luodi2024learning}.
However, to the best of our knowledge, none of these methods have stability guarantees, which ensure that if two node/time pairs `behave the same' in the network, their representation is the same up to noise. 
Stability allows for the comparison of embeddings over time because the embedding space has a consistent interpretation.

\begin{figure*}[h]
\centering
  \includegraphics[width=\textwidth]{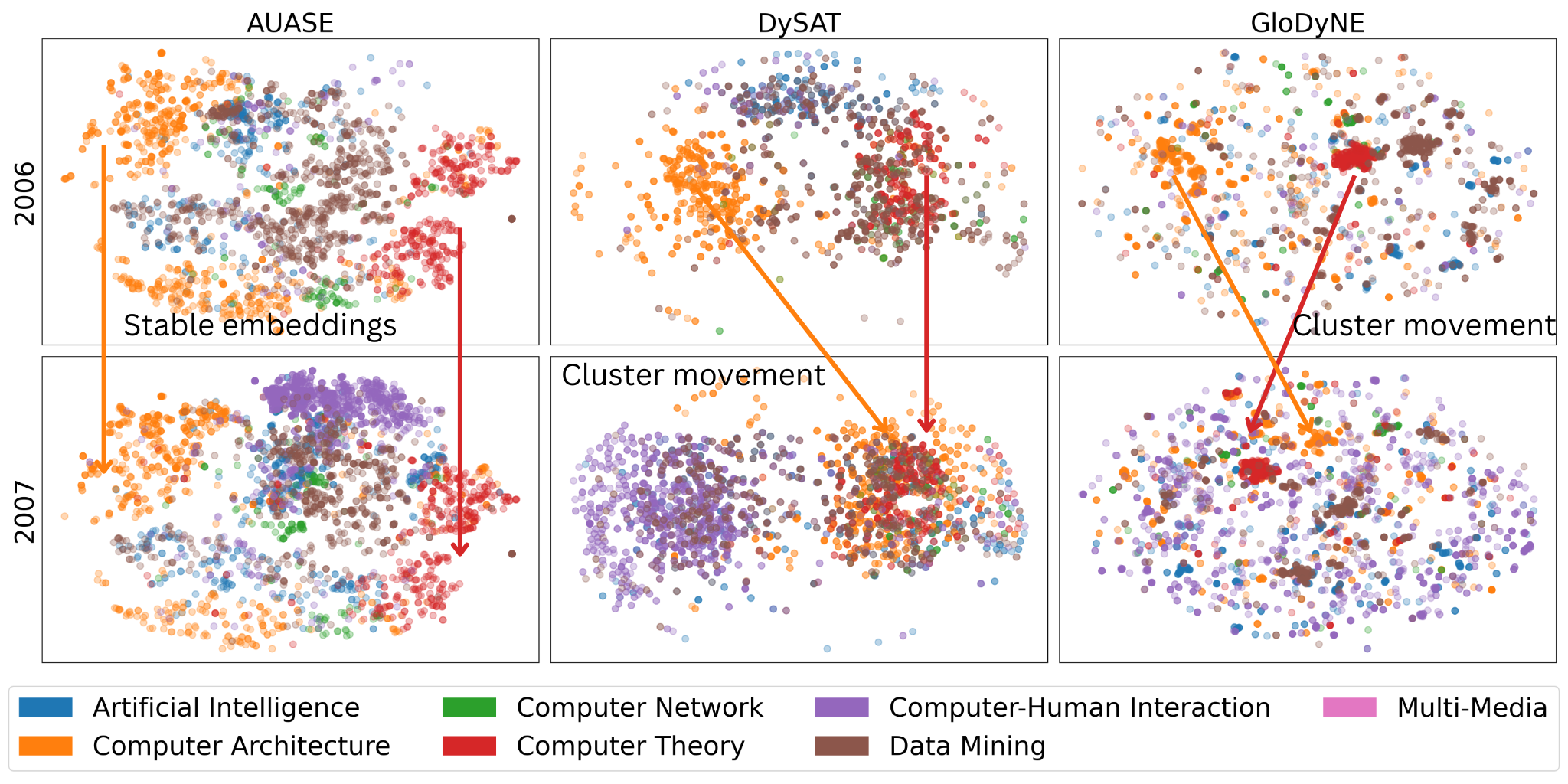}
  \caption{AUSE provides stable embeddings, meaning that nodes that behave the same at different times receive the same embedding illustrated with two-dimensional t-SNE visualisations of node embeddings (in a shared temporal space) of the DBLP datasets. Colours correspond to community membership.}\label{fig:DBLP}
\end{figure*}

Attributed unfolded adjacency spectral embedding (AUASE) is a framework for unsupervised dynamic attributed network embedding with stability guarantees. AUASE builds on unfolded adjacency spectral embedding (UASE) \citep{jones2021multilayer}, an unsupervised spectral method for dynamic network embedding with stability guarantees \citep{gallagher2022spectral}. However, UASE is only suited to unattributed networks. To include node attributes, AUASE combines the adjacency matrix and a covariate matrix, with edge weight proportional to covariate values, into an attributed adjacency matrix. 

To analyse the statistical properties of AUASE, we define an attributed dynamic latent position model combining a latent position dynamic network model \citep{gallagher2022spectral} with a dynamic covariate model. We show that AUASEs converge asymptotically to the noise-free embedding, which is essential to demonstrate AUASE's stability properties. 

Figure~\ref{fig:DBLP} shows the embeddings of AUASE and two comparable methods - visually explaining the difference between stable and unstable embeddings. The astute reader will be wondering why leading dynamic graph methods could have such a visually poor internal representation. There are a wide class of problems - for which graph neural networks (GNNs) were developed - that can be expressed as a supervised learning problem, and the classifier learns a separate map for each time point, for instance, node classification and link prediction \citep{rossi2020temporal,pareja2020evolvegcn}. However, many problems are not of this form and require \emph{unsupervised stable} embeddings.  Examples include but are not limited to: group discovery in purchasing networks for marketing purposes \citep{Palla_2007}, anomaly detection in financial or cybersecurity networks \citep{ranshous2015anomaly} and protein function prediction in dynamic biological networks \citep{klein2012structural, yue2020graph}. 

We compare AUASE to state-of-the-art unsupervised attributed dynamic embeddings on the tasks of node classification and link prediction as a means to demonstrate AUASE embedding's validity and the value of stability. In extensive experiments on four large real-world datasets AUASE considerably outperforms other methods, which in many cases perform worse than baseline guessing precisely because they lack embedding stability.
 
\section{Theory \& Methods}\label{sec:theory}

Let $\mathcal{G} = (\mathcal{G}^{(1)}, \dots, \mathcal{G}^{(T)})$ represent a dynamic sequence of $T$ node-attributed networks each with nodes $[n] = \{1, \dots , n\}$. At time $t \in [T]$, we denote $\mathcal{G}^{(t)} = (\A^{(t)}, \C^{(t)})$, where $\A^{(t)} \in \{0,1\}^{n \times n}$ represents the network structure using a binary adjacency matrix and $\C^{(t)} \in \mathbb{R}^{n \times p}$ represents the time-varying node attributes as covariates. The aim is to obtain a low-dimensional embedding $\{\hat{\Y}_\A^{(t)}\}_{t \in [T]}$ representing the network and covariate behavior of each node across time. 

\subsection{Attributed Unfolded Adjacency Spectral Embedding}

We present our attributed embedding procedure, which extends unfolded adjacency spectral embedding (UASE).

For each time period $t \in [T]$, we incorporate the covariates $\+C^{(t)}$ into the adjacency matrix $\+A^{(t)}$ by including them as $p$ attribute nodes in the network. Nodes are connected to these $p$ attribute nodes with edge weight proportional to the corresponding covariate value, producing the augmented adjacency matrix,
\begin{align*}
    \A_C^{(t)} = \begin{bmatrix}
    (1-\alpha)\A^{(t)} & \alpha\C^{(t)} \\
    \alpha\C^{(t)\top} & \mathbf{0}_{p \times p}
    \end{bmatrix} \in \mathbb{R}^{(n+p) \times (n+p)},
\end{align*}
where the hyperparameter $\alpha \in [0,1]$ balances the relative contributions of $\A^{(t)}$ and $\C^{(t)}$, fixed for all $t \in [T]$.

Given the augmented adjacency matrices $\A_C^{(t)}$, the procedure continues analogously to UASE. Hence, for $\alpha =0$, AUASE reduces to UASE. The unfolded attributed adjacency matrix $\A_C$ is constructed by concatenating the attributed adjacency matrices, and the dynamic spectral embedding is obtained using the output of the $d$-truncated SVD. This procedure is detailed in Algorithm \ref{AUASE}.

\begin{algorithm}[h]
\caption{Attributed unfolded adjacency spectral embedding (AUASE).}\label{AUASE}
\begin{algorithmic}[1]
\Require Attributed dynamic network $\mathcal{G}$, embedding dimension $d$, hyperparameter $\alpha \in [0,1]$.

\State Construct the attributed adjacency matrix $\A_C^{(t)}$.

\State Construct the unfolded attributed adjacency matrix
\begin{equation*}
    \A_C = (\A_C^{(1)} \mid \cdots \mid \A_C^{(T)}).
\end{equation*}
\State Compute the $d$-truncated SVD of \begin{equation*}
    \A_C \approx \U_{\A} \+\Sigma_{\A}\V_{\A}^\top.
\end{equation*}

\State Divide $\V_{\A}^\top$ into $T$ blocks \begin{equation*}\V_{\A} = (\V_{\A}^{(1)} \mid  \cdots \mid \V_{\A}^{(T)}).\end{equation*}
\State Define the attributed unfolded adjacency spectral embeddings (AUASE) as 
\begin{equation*}\hat{\Y}_\A^{(t)}= (\Y^{(t)}_{\A})_{1:n}=(\V_{\A}^{(t)}\+\Sigma_{\A}^{1/2})_{1:n}^\top,\end{equation*}

\Ensure Node embeddings $\{\hat{\Y}_\A^{(t)}\}_{t \in [T]}$.
\end{algorithmic}
\end{algorithm}

We focus exclusively on the rows of $\Y_\A^{(t)}$ corresponding to nodes in the dynamic network $\mathcal{G}$. While AUASE produces dynamic embeddings for the $p$ attribute nodes, we do not consider their asymptotic properties as we assume a fixed number of attributes.

\subsection{Attributed Dynamic Network Model} To analyse the asymptotic properties of the embeddings, we propose a latent position model for dynamic attributed networks. The use of dynamic latent position models to study dynamic networks is well established \citep{jones2021multilayer, sewell2015latent, lee2011latent, kim2018reviewdynamicnetworkmodels}. To incorporate node attributes, we combine a non-attributed dynamic network model \citep{gallagher2022spectral} with a dynamic covariate model using the same dynamic latent positions.

\begin{definition}[Dynamic network model]\label{def:DNM}
Conditional on latent positions $Z_i^{(t)} \in \mathcal{Z}$, the sequence of symmetric matrices $\+A^{(1)}, \ldots, \+A^{(T)} \in \mathbb{R}^{n \times n}$ is distributed as a dynamic latent position network model, if, for all $i < j$,
\begin{align*}
    \A_{ij}^{(t)} \mid Z_i^{(t)},Z_j^{(t)} &\stackrel{\text{ind}}{\sim} H(Z_i^{(t)}, Z_j^{(t)}),
\end{align*}
where $H$ is a symmetric real-valued distribution function, $H(Z_1, Z_2) = H(Z_2, Z_1)$, for all $Z_1, Z_2 \in \mathcal{Z}$.
\end{definition}

\begin{definition}[Dynamic covariate model]\label{def:DCM}
Conditional on latent positions $Z_i^{(t)} \in \mathcal{Z}$, the sequence of covariate matrices $\+C^{(1)}, \ldots, \+C^{(T)} \in \mathbb{R}^{n \times p}$ is distributed as a dynamic latent position covariate model, if, for all $i \in [n]$, $\ell \in [p]$,
\begin{align*}
    \C_{i\ell}^{(t)} \mid Z_i^{(t)} &\stackrel{\text{ind}}{\sim} f_\ell(Z_i^{(t)}),
\end{align*}
where $f_\ell(Z)$ is a real-valued distribution for all $Z \in \mathcal{Z}$.
\end{definition}

Our goal is to describe the attributed adjacency matrices $\+A_C^{(t)}$ as a special type of dynamic network model (Definition \ref{def:DNM}). To include node attributes, we add fixed latent positions $Z_{n+\ell}^{(t)} = i_\ell$ representing an index for each covariate such that $\mathcal{I} = \{i_\ell: \ell \in [p] \}$ is disjoint from the possible node latent positions $\mathcal{Z}$. For convenience, we denote $f_{i_\ell}(Z) = f_\ell(Z)$.

\begin{definition}[Attributed dynamic latent position model]\label{def:ADCM}
Conditional on latent positions $Z_i^{(t)} \in \mathcal{Z}$ and fixed latent positions $Z_{n+\ell}^{(t)} = i_\ell$, the sequence of symmetric matrices $\+A_C^{(1)}, \ldots, \+A_C^{(T)} \in \mathbb{R}^{(n+p) \times (n+p)}$ is distributed as an attributed dynamic latent position network model with sparsity parameter $\rho > 0$, if,

for all $i < j$,
\begin{align*}
    \A_{ij}^{(t)} \mid Z_i^{(t)},Z_j^{(t)} &\stackrel{\text{ind}}{\sim} H(Z_i^{(t)}, Z_j^{(t)}),
\end{align*}
where $H$ is a symmetric real-valued distribution function satisfying
\begin{align*}
    \lefteqn{H(Z_i, Z_j) =} \\
    &\begin{cases}
        (1-\alpha)\t{Bernoulli}( \rho f(Z_i, Z_j)) & Z_i, Z_j \in \mathcal{Z}, \\
        \alpha \rho^{1/2} f_{Z_j}(Z_i) & Z_i \in \mathcal{Z}, Z_j \in \mathcal{I}, \\
        \delta_0 & Z_i, Z_j \in \mathcal{I},
    \end{cases}
\end{align*}
\end{definition}
where $\delta_0$ represents the Dirac delta function with mass at zero, and $\+X \sim \alpha f$ is shorthand for $\+X / \alpha \sim f$.

Definition~\ref{def:ADCM} includes the scenario where the covariates do not change over time, for all nodes, $\+C_{i \ell}^{(t)} = \+C_{i \ell}$ for all $t \in [T]$. To show this, we construct latent positions which includes the fixed covariate, $Z'_i = (Z_i, \+C_{i \ell})$. The dynamic covariate model is then a deterministic function of this new latent position, $f_\ell(Z'_i) = \+C_{i \ell}$. Multiple covariates can be included this way.

\subsection{Theoretical Results}
In this section, we present theoretical results that describe the asymptotic behaviour of the dynamic embedding $\hat{\Y}^{(t)}$ given the following model assumptions:

\begin{assumption}[Low-rank expectation]\label{low rank ass}
There exist maps $\phi: \mathcal{Z}^T \cup \mathcal{I}^T \rightarrow \mathbb{R}^d $ and $\phi_t: \mathcal{Z} \cup \mathcal{I} \rightarrow \mathbb{R}^d $ for all $t \in [T]$ such that 
\begin{equation*}
    \mathbb{E} [(\A_C^{(t)})_{ij} \mid Z_i, Z_j] = \phi(Z_i) \phi_t(Z_j^{(t)})^\top,
\end{equation*}
where $Z_i = \{ Z_i^{(1)}, \ldots, Z_i^{(T)} \}$.
\end{assumption}

\begin{assumption}[Singular values of $\P_C$]\label{sing values P ass}
The $d$ non-zero singular values of the mean unfolded adjacency matrix $\P_C = \mathbb{E}(\A_C)$ satisfy 
\begin{equation*}
    \sigma_i(\P_C) = \Theta(T^{1/2}\rho n)
\end{equation*}
 for all $i \in [d]$ with high probability as $n \rightarrow \infty$.
\end{assumption}

\begin{assumption}[Finite number of covariates]\label{n cov ass} $p = \O(1)$.
\end{assumption}

% to check 
\begin{assumption}[Network sparsity]\label{sparsity ass}
    The sparsity factor $\rho$ satisfies \begin{equation*}
        \rho  = 
        \omega(n^{-1} \log^k(n) )
    \end{equation*}
    for some constant $k$.
\end{assumption}

\begin{assumption}[Subexponential tails]\label{exp tail}
There exists a constant $\beta >0$ such that
\begin{equation*}
    \mathbb{P}\left(\lvert  f_l(Z_i^{(t)}) - \mathbb{E}[f_l(Z_i^{(t)})] \rvert > x \mid Z_i^{(t)}\right) \leq \t{exp}\left(-\beta x\right)
\end{equation*}
\end{assumption}

Assumption~\ref{low rank ass} allows us to define a canonical choice for the maps $\phi$ based on the mean attributed unfolded adjacency matrix $\P_C = \mathbb{E}[\A_C]$. The low-rank assumption states that the $d$-truncated SVD $\P_C = \U_{\P}\+\Sigma_{\P} \V_{\P}$ is exact and we define canonical choice for the maps,
\begin{align*}
    \phi(Z_i) &= (\U_{\P}\+\Sigma_{\P}^{1/2})_i \in \mathbb{R}^d, \\
    \phi_t(Z_i^{(t)}) & = (\V_{\P}^{(t)} \+\Sigma_{\P}^{1/2})_i= (\Y_\P^{(t)})_i \in \mathbb{R}^d.
\end{align*}
In particular, we will demonstrate how the AUASE output $\hat{\Y}_\A^{(t)}$ is a good approximation of the noise-free embedding $\hat{\Y}_\P^{(t)} = (\Y_\P^{(t)})_{1:n} $ defined by the map $\phi_t$. 

Assumption~\ref{sing values P ass} is a technical condition which ensures that the growth of the singular values of $\P_C$ is regulated. Assumption~\ref{n cov ass} assumes a fixed number of covariates which alongside Assumption~\ref{sparsity ass} ensures that the networks are dense enough to recover the latent positions. Versions of these assumptions are required to analyse the multilayer random dot product graph \citep{jones2021multilayer}. 
These assumptions are not limiting as for finite $n$ (i.e., any practical application) scaling $\alpha$ would mitigate any imbalance between sparsity in the graph and the covariates.

Assumption~\ref{exp tail} prevents large attribute values, which would otherwise ruin the structure in an embedding. It is satisfied trivially for bounded distributions like the Bernoulli and Beta distributions and many other
common distributions like the exponential and Gaussian distributions. An equivalent condition is established in the weighted generalised random dot product graph model \citep{gallagher2023spectralembeddingweightedgraphs}.

Our main result shows that $\hat{\Y}_\A^{(t)}$ is a good approximation of some invertible linear transformation of $\hat{\Y}_\P^{(t)}$, as the size of the largest error tends to zero as the number of nodes grows. Let $\lVert \cdot \rVert_{2\rightarrow \infty}$ denote the two-to-infinity norm \citep{cape2019two}, the maximum row-wise Euclidean norm of a matrix.

\begin{theorem}[Uniform consistency]\label{2toinf}
   Under Assumptions~\ref{low rank ass}-\ref{exp tail} there exists a sequence of orthogonal matrices $\W = \W_n \in \mathbb{O}(d)$ such that, for all $t \in [T]$,
   \begin{align*}
       \lVert{\hat{\Y}}_\A^{(t)} - \hat{\Y}^{(t)}_{\P} \W\rVert_{2\rightarrow \infty} &= \OP\left(\frac{\log^{1/2}(n)\,r_\alpha(1,1/\beta)}{T^{1/4}\rho^{1/2}n^{1/2}}\right),
   \end{align*}
where for $x,y \in \mb{R}$, we define $r_\alpha(x,y) = (1-\alpha)x + \alpha y$, and $X_n =\OP(a_n)$ means that $X_n/a_n$ is bounded in probability \cite{janson2011probability}.
\end{theorem}

Theorem \ref{2toinf} has important methodological implications similar to the equivalent theorem in \cite{jones2021multilayer}. Uniform consistency in the two-to-infinity norm implies that subsequent statistical analysis is consistent up to rotation.

\subsection{Stability Properties}\label{sec:stability}
\cite{xue2022dynamic} argues that dynamic network embeddings are desirable to preserve the network's local structure and the long-term dynamic evolution. The survey observes that most dynamic network embeddings possess only one of these properties, but not both. In this section, we show that AUASE captures both local and global network structures, which we refer to as spatial and temporal stability, respectively.

We say that node/time pairs $(i,s)$ and $(j,t)$ are exchangeable if they `behave the same' within the dynamic network. For the attributed dynamic latent position model (Definition~\ref{def:ADCM}) this is equivalent to stating, for all $Z \in \mathcal{Z} \cup \mathcal{I}$,
\begin{align*}
    H(Z_i^{(s)}, Z) &= H(Z_j^{(t)}, Z).
\end{align*}
It is desirable for exchangeable node/time pairs $(i,s)$ and $(j,t)$ to have equal embeddings $\hat{\Y}_i^{(s)}$ and $\hat{\Y}_j^{(t)}$ up to noise. If two exchangeable nodes behave consistently at a fixed time ($i \ne j, s = t$), we say the embedding has spatial stability. If a node behaves consistently between two exchangeable time points ($i = j, s \ne t$), we say the embedding has temporal stability.

The following lemma shows that AUASE has both spatial and temporal stability.

\begin{lemma}[AUASE stability] \label{stab-lemma}
Given exchangeable node/time pairs $(i,s)$ and $(j,t)$, the embeddings $(\hat{\Y}_\A^{(s)})_i$ and $(\hat{\Y}_\A^{(t)})_j$ are asymptotically equal,
\begin{align*}
    \lVert{ (\hat{\Y}_\A^{(s)})_i - (\hat{\Y}_\A^{(t)})_j \rVert} &=
    \OP\left(\frac{\log^{1/2}(n)\,r_\alpha(1,1/\beta)}{T^{1/4}\rho^{1/2}n^{1/2}}\right).
\end{align*}
\end{lemma}

Spatial and temporal stability is crucial for downstream machine-learning tasks. Training a predictive model using embedding data is only meaningful if the embedding space has a consistent interpretation. Sections~\ref{Sec:Example}, \ref{sec:node_class},   and \ref{sec:link_pred} demonstrate problems when using dynamic embedding techniques that do not have spatial or temporal stability. To the best of our knowledge, AUASE is the only existing attributed dynamic embedding which satisfies both these desirable stability properties.

\subsection{Parameter Selection}

The AUASE procedure detailed in Algorithm \ref{AUASE} requires two additional input parameters; the embedding dimension $d$ and the attributed adjacency matrix weight $\alpha$. This section outlines approaches for selecting these parameters.

In the theory it is assumed that the embedding dimension $d$ is known, however in practice it will need to be estimated. For a given $\alpha$, the truncated SVD of $\+A_C$ is computed up to some maximum embedding dimension and $d$ is determined by a singular value threshold, for example, using profile likelihood \cite{zhu2006automatic} or ScreeNOT \cite{donoho2023screenot}.

The hyperparameter $\alpha$ balances the contributions of the network and covariates in the AUASE procedure. The truncated SVD gives the best low-rank approximation of $\+A_C$ with respect to the Frobenius norm. By increasing $\alpha$, the truncated SVD is incentivised to minimise the squared error terms corresponding to the scaled covariate matrices in $\+A_C$.

For a downstream task, such as node classification, cross-validation can be used to choose $\alpha$. For unsupervised learning, one can compare embedding quality metrics \citep{tsitsulin2023unsupervised}. We show that practical performance is robust to the choice of $\alpha$ in Section~\ref{sec:alpha}.

\section{Experiments}

\subsection{Method Comparison} \label{Sec:Methods}

We compare AUASE to the following baselines: 
\begin{itemize}
    \item DRLAN \citep{liu2020dynamic}: Efficient framework incorporating an offline and an online network embedding model. Hyperparameter $\beta$ weights network features and attribute contribution.
    \item DySAT \citep{sankar2020dysat}: Learns node embeddings through self-attention layers and temporal dynamics.
    \item GloDyNE \citep{hou2020glodyne}: Dynamic but not attributed, and updates node embeddings for a selected subset of nodes that accumulate the largest topological changes over subsequent network snapshots.
    \item DyRep \citep{trivedi2019dyrep}: Inductive GNN-based dynamic graph embedding method, it supports (non-time-varying) node attributes. 
    \item CONN \citep{tan2023collaborative}: Attributed but static GNN that explicitly exploits node attributes. 
    % The hyperparameter $\alpha$ controls the information diffusion between the network and attributes.
\end{itemize}

To the best of our knowledge, DRLAN and DySAT are the only unsupervised time-varying attributed dynamic embeddings with available source code, hence directly comparable to AUASE. The motivations for the exclusion of other comparable methods are in Appendix \ref{sec:meth_selec}.

For the sake of completeness, we also include a dynamic but not attributed method (GloDyNE), a dynamic embedding with non-time-varying attributes (DyRep) and a static method (CONN). However, these methods are not directly comparable with AUASE and, therefore, are not suited for tasks which require a stable dynamic attributed embedding on data with time-varying attributes. 

\subsection{Implementation and Efficiency}

Both AUASE and UASE can be easily implemented with the \textit{pyemb} Python package\footnote{\url{https://pyemb.github.io/pyemb/html/index.html}}. The code for the reproducing experiments is available at this GitHub repository \footnote{\url{https://github.com/emmaceccherini/AUASE}}. 

The computational times to compute the dynamic embeddings of four large real-data networks for each baseline are reported in Table \ref{tab:static_sc}. The runtime of AUASE implemented with the \textit{pyemb} Python package is orders of magnitude faster than competing methods, showing that AUASE is highly efficient for large graphs. 

We do not provide a comprehensive time and space complexity analysis as AUASE is essentially a sparse truncated SVD whose space and time complexity has been extensively studied. The space complexity in the sparsest regime we consider is $O(T n\log^k(n))$. 

The truncated SVD of $\A_C$
 can be computed using the Augmented Implicitly Restarted Lanczos Bidiagonalization algorithm \citep{baglama2005augmented} implemented in the \textit{irlba} package in R, or the \textit{irlbpy} in Python. Although the exact time complexity of this algorithm has not been studied theoretically, it is known that the time complexity of more general Lanczos-type algorithms for computing the rank 
 truncated SVD is  $O(Nd)$ \citep{tomas2023fast}, in the sparsest regime we consider, this is $O(T n\log^k(n) d)$. The author of the \textit{irlba} package demonstrates its speed by performing a simulated experiment in which they compute the first 2 singular vectors of a sparse 10M x 10M matrix with 1M non-zero entries which takes approximately 6 seconds on a computer with two Intel Xeon CPU E5-2650 processors (16 physical CPU cores) running at 2 GHz equipped with 128 GB of ECC DDR3 RAM \footnote{\url{https://bwlewis.github.io/irlba/comparison.html}}.

\subsection{Simulated Example} \label{Sec:Example}

\begin{figure*}[ht]
\centering
  \includegraphics[width=\textwidth]{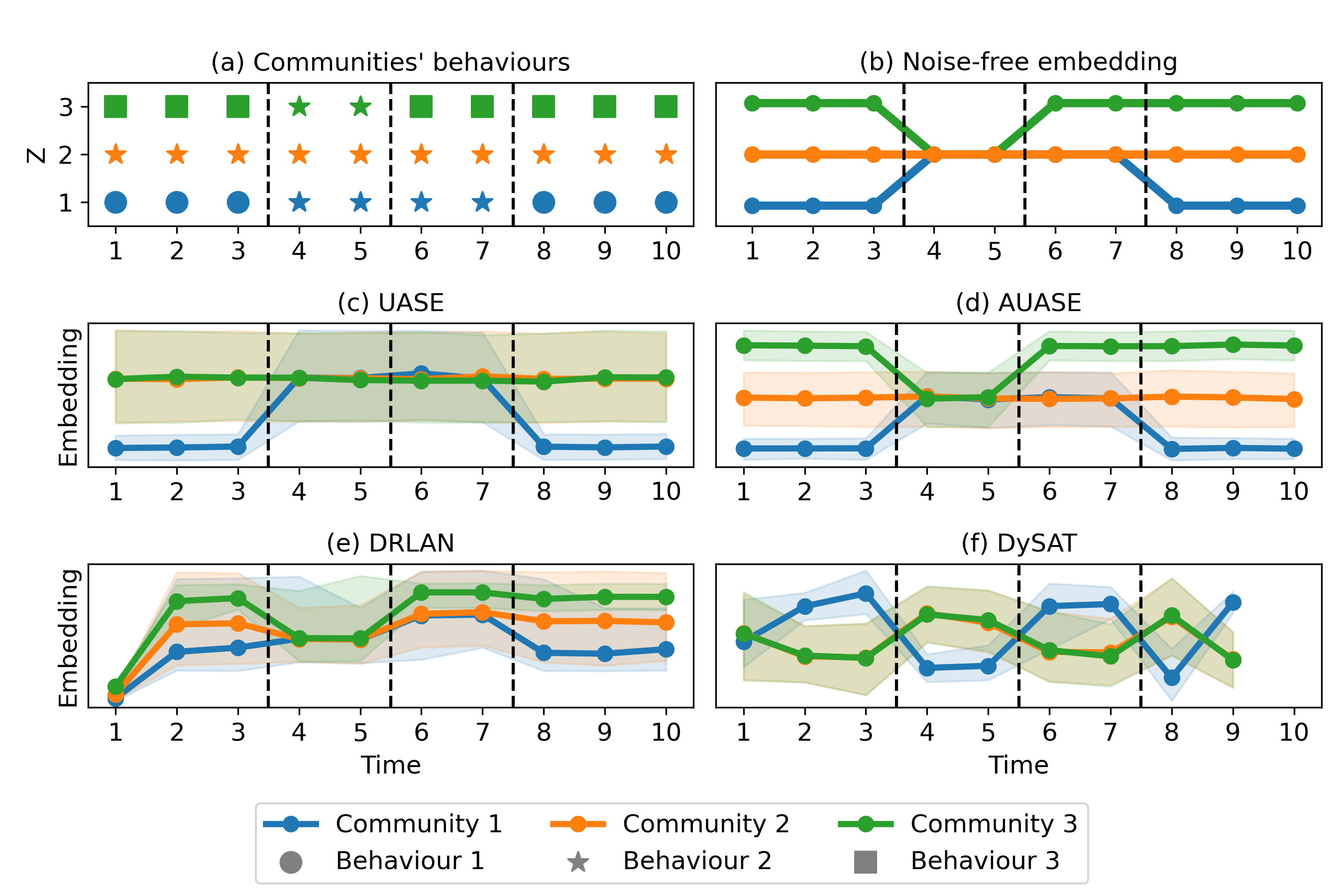}
  \caption{(a) Representation of global community (colour) and local behaviours (shape) of the synthetic attributed dynamic network. (b)-(f) One-dimensional UMAP visualisation of different node embedding techniques. The coloured lines show the mean embedding for each community with a 90\% confidence interval. The y-axis should only be interpreted as an illustration of community separation, rotated to best align the ground truth.}
  \label{Figure1}

\end{figure*}

To illustrate the properties of AUASE we simulate a sequence of attributed networks using a dynamic stochastic block model with three communities with time-varying node attributes.

We assume three possible behaviours denoted $Z_i^{(t)} \in \{1,2,3\}$ and define the attributed dynamic latent position model in Definition~\ref{def:ADCM},
\begin{align*}
    \+A_{ij}^{(t)} \mid Z_i^{(t)}, Z_j^{(t)} &\stackrel{\text{ind}}{\sim} \t{Bernoulli}(\+B_{Z_i^{(t)}, Z_j^{(t)}}),\\
    \+C_{i\ell}^{(t)} \mid Z_i^{(t)} &\stackrel{\text{ind}}{\sim} \t{Normal}(\+D_{Z^{(t)}_i, \ell}, \+I),
\end{align*}
where the latent position indexes into the community edge probability matrix $\+B \in \mathbb{R}^{3 \times 3}$ and  the mean attribute matrix $\+D \in \mathbb{R}^{3 \times p}$,
\begin{equation*}
    \+B = \begin{pmatrix}
    p_1 & p_0 & p_0\\
    p_0 & p_0 & p_0\\
    p_0 & p_0 & p_0
\end{pmatrix}, \quad 
    \+D = \begin{pmatrix}
    \bm{\mu}_1 \\
    \bm{\mu}_1 \\
    \bm{\mu}_2
\end{pmatrix}.
\end{equation*}
Note that in this experiment it is impossible to distinguish $Z_i^{(t)} = 2$ and $Z_i^{(t)} = 3$ using the adjacency matrices alone. Likewise, it is impossible to distinguish between $Z_i^{(t)} = 1$ and $Z_i^{(t)} = 2$ using the covariate matrices alone. Both are required to determine all three types of behaviour.

We construct the attributed dynamic latent position model with $n = 1000$ nodes and $p=150$ covariates over $T=10$ time points. Nodes are assigned with equal probability to three communities, depicted by the colours in Figure~\ref{Figure1}a, corresponding to a trajectory of local behaviours $(Z^{(1)}_i, \dots, Z^{(T)}_i)$, represented by the shapes in Figure~\ref{Figure1}a. The edge probabilities are $p_0 = 0.5$, $p_1 = 0.3$ and mean attributes $\bm{\mu}_1, \bm{\mu}_2 \in \mathbb{R}^p$ are detailed in Appendix~\ref{sec:extra_3}.

Figure~\ref{Figure1}b demonstrates the behaviour of the three communities we would like the dynamic embedding techniques to recover, shown through the noise-free embedding of $\+P_C = \E(\+A_C)$. For instance, for $t \in \{4,5\}$, all three communities have $Z^{(t)}_i = 2$ (star), so the embeddings are the same due to spatial stability. After $t = 7$ community $Z_i = 1$ (blue) switches from latent position $Z_i^{(t)} = 2$ (star) to $Z_i^{(t)} = 1$ (circle), so the embedding returns to its previous position due to temporal stability.

We compute the dynamic embedding using UASE, AUASE with $\alpha = 0.2$, DRLAN and DySAT into $d=3$ dimensions, the known number of communities. Further details and experiments regarding the choice of these parameters and specific algorithm hyperparameters can be found in Appendix \ref{sec:extra_3}.

Figures~\ref{Figure1}c-\ref{Figure1}f show the dynamic network embedding for these four techniques. For visualisation, we reduce the three-dimensional embeddings to one dimension with UMAP \cite{mcinnes2020umapuniformmanifoldapproximation}. The solid line shows the mean embedding for each community with a 90\% confidence interval.

Only AUASE is able to fully recover the underlying structure of the attributed network, displaying both spatial and temporal stability as predicted by the theory. UASE cannot detect the differences between communities 2 and 3, as we expected, as these only differ in their covariates, which is not considered by the technique.

DRLAN captures the dynamic patterns from both edge and attribute information, although it is unable to produce a sensible embedding at time $t=0$ and has wide confidence intervals. The stability is also very sensitive to the choice of hyperparameter (see Appendix~\ref{sec:extra_plot3}). DySAT is clearly unstable, for instance, the switching of embeddings between times $t=7$ and $t=8$. Moreover, the embeddings of communities 2 and 3 appear merged for the whole time frame as observed with UASE. This occurs because DySAT, like most existing neural network methods, only incorporates node attributes as input to the first layer \citep{dwivedi2022benchmarkinggraphneuralnetworks}.

\subsection{Datasets} \label{Sec:RealData}

As discussed in Section~\ref{sec:stability}, stability is a desirable intrinsic property of a dynamic embedding essential for any downstream analysis. To provide an objective measure of its value, we evaluate the learned node representations on the downstream tasks node classification and link prediction for the following four attributed dynamic networks. Further details about the construction of the adjacency matrices $\+A^{(t)}$ and covariate matrices $\+C^{(t)}$ are given in Appendix~\ref{sec:extra_4}.
\begin{itemize}
    \item DBLP\footnotemark{} \citep{ley2009dblp}: Co-authorship network consisting of bibliography data from computer science publications. Nodes represent authors with edges denoting co-authorship. Attributes are derived from an author's abstracts published in each time interval. For each interval an author has one of seven labels based on their most frequent publication area.

    \item ACM\footnotemark[\value{footnote}]: Co-authorship network similar to DBLP. 

    \footnotetext{\url{https://www.aminer.cn/citation}}
    \item Epinions\footnote{\url{http://www.epinions.com}} \citep{tang2013exploiting}: Trust network among users of a social network for sharing reviews. Nodes represent users with edges denoting a mutual trust relationship between users. Attributes are derived from users' reviews in each time interval. For each interval, a user has one of 22 labels based on their frequent review category. 

    \item ogbn-mag \footnote{\url{https://ogb.stanford.edu/docs/nodeprop}} \citep{wang2020microsoft}: Co-authorship network from the Microsoft Academic Graph. We construct 10 adjacency matrices and attributes matrices where authors are the nodes. The tasks we perform are at author-level not paper-level (usual for ogbn-mag) which creates much harder tasks (see Appendix~\ref{sec:extra_4} for more details).  
\end{itemize}
Dataset statistics are shown in Table~\ref{tab:datasets}.
\begin{table}[h]
    \centering
    \caption{Attributed dynamic network statistics for the datasets showing the number of nodes $n$, the number of covariates $p$, the number of time intervals $T$, and the number of node labels $K$.}
    \label{tab:datasets}
     \resizebox{\columnwidth}{!}{
    \begin{tabular}{c|c|c|c|c}
 \hline
 Datasets & \makecell{Nodes - $n$} & \makecell{Attributes - $p$} & \makecell{Intervals - $T$} & \makecell{Labels - $K$}\\
 \hline
 DBLP & 10092&4291 & 9&7 \\ 
 ACM & 34339 &6489&15& 2\\
 Epinions &15851 &7726&11&22\\
 ogbn-mag &479979 &128 &10&349\\
 \hline
    \end{tabular}
}
\end{table}

The embedding dimension is selected using ScreeNOT \citep{zhu2006automatic} giving $d=15$, $d=29$ and $d=22$ for DBLP, ACM and Epinions, respectively. The one exception is DySAT on the ACM dataset where $d =32$. For ogbn-mag we choose $d=50$ based on the scree plot. We use degree correction for AUASE and UASE embeddings \citep{passino2022spectral}. 

The number of nodes in ogbn-mag is an order of magnitude larger than in the other datasets, therefore only UASE, AUASE and DRLAN are computationally feasible. More details about the computational barriers of DySAT, GloDyNE, DyRep and CONN are given in Appendix~\ref{sec:extra_4}.

Figure~\ref{fig:DBLP} shows the dynamic embedding using AUASE, DySAT and GloDyNE for the DBLP dataset for two consecutive time intervals. We reduce each embedding into two-dimensions using t-SNE \cite{van2008visualizing} jointly over all time intervals to preserve stability. The DySAT and GloDyNE embeddings find community structures, but they appear to have moved over time showing a lack of temporal stability. Conversely, AUASE shows both types of stability, for instance, authors in Computer Architecture and Computer Theory staying approximately in the same position over time.

\begin{figure*}[h]

\centering

    \includegraphics[width=2.0\columnwidth]{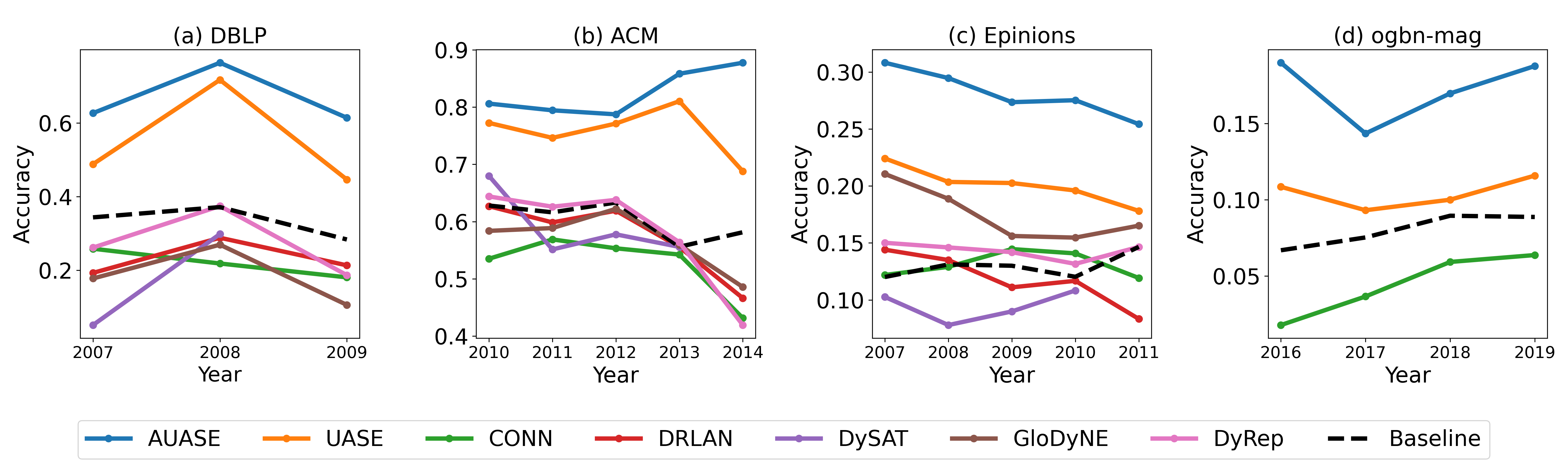} 

    \caption{Node label classification accuracies for DBLP, ACM, Epinions and ogbn-mag datasets for each time interval in the corresponding test datasets. The dynamic embedding techniques outlined in Section~\ref{Sec:Methods} are solid lines, and different colors represents different methods. A basic baseline predicting the most common label is dashed.}\label{fig:dyn_acc}

\end{figure*}
\subsection{Node Classification}\label{sec:node_class}

Node classification aims to predict label categories for current nodes using the historical node embedding data. We train a classifier using XGBoost \citep{chen2016xgboost} on the first 65\% of time intervals for all the datasets and test on the remaining data. We reserve 10\% of the training data for validation to select the weight hyperparameters via cross-validation for AUASE, DRLAN and CONN. Other hyperparameters are set to the default values provided by the original authors.

Figure~\ref{fig:dyn_acc} reports the classification accuracy for each test year, while Table~\ref{tab:static_sc} reports the mean accuracy over all test data. Other metrics were considered and gave equivalent results (see Appendix~\ref{sec:extra_4}, Tables \ref{tab:acm_results}, \ref{tab:epin_results} and \ref{tab:dblp_results}). Because label prediction requires predicting current embeddings, the stability illustrated in Figure \ref{fig:DBLP} allows AUASE to dramatically outperform baselines in terms of accuracy  - 40\% over DyRep on DBLP, 23\% over DySAT on ACM, 11\% over GloDyNE on Epinions and 12\% over DRLAN on ogbn-mag. 
AUASE also performs noticeably better than UASE, suggesting that the inclusion of covariate information is improving the quality of the dynamic embedding. Stability is clearly driving this performance increase, as Figure~\ref{fig:dyn_acc} shows that
no unstable dynamic method consistently outperforms the baseline classifier of predicting the most common class.

\begin{table*}[h]
    \centering
    \caption{AUC for link prediction, average accuracy for node classification and time to compute the dynamic embedding for each baseline.  Errors show the 90\% CI, the best performance is {\bf bold}, second \underline{underlined}.}
    \label{tab:static_sc}
     % \resizebox{\columnwidth}{!}{   
    \begin{tabular}{c|c|c|c|c|c} 
      \hline
    Method & Task & DBLP & ACM &Epinions & ogbn-mag \\
    \hline
    
    \multirow{3}{*}{CONN}&
    Link Pred. &$0.741\pm0.007$ & $0.771\pm0.008$&$0.602\pm0.072$ & -\\
    & Node Class. & 0.219&0.526& 0.131 & -\\
    & Time & $\approx$1h & $\approx$1h & $\approx$1h & -\\
    \hline
    \multirow{3}{*}{GloDyNE}
    & Link Pred. & $0.883\pm 0.004$ & $\mathbf{0.907}\pm0.005$& $0.729\pm0.016$& -\\
    & Node Class. & 0.185 &0.568 &   0.175 & -\\   
     & Time& $\approx$10m&$\approx$1h &$\approx$1h& -\\
    % \hline
    % \multirow{3}{*}{ContinualGNN} & Accuracy &\\ &F1 &\\&Time & $\approx$20m & $>$ 1h\\
    \hline
    \multirow{3}{*}{DRLAN}
    & Link Pred. & $0.575\pm0.018$ & $ 0.641\pm0.015$&$0.626\pm0.081$ & $\underline{0.939}\pm0.009$ \\
    & Node Class. &0.232 & 0.573& 0.118 & 0.044 \\           
    & Time & $\approx$1s &  $\approx$1s & $\approx$1s & $\approx$ 10m\\
    \hline
    \multirow{3}{*}{DyRep} 
    & Link Pred. & $0.511\pm0.014$ & $0.552\pm0.010$ &$0.548\pm0.022$ & -\\
    & Node Class. & 0.272& 0.578 &0.143 & -\\
    & Time & $\approx$5h & $\approx$40h& $\approx$48h & -\\
    \hline
    \multirow{3}{*}{DySAT} 
    & Link Pred. &  $0.802\pm0.018$& $0.758\pm0.023$&$0.596\pm0.107$& -\\
    & Node Class. & 0.175& 0.591& 0.095& -\\
    & Time & $\approx$1h & $\approx$5h&$\approx$3h & -\\
    \hline
    \multirow{3}{*}{UASE} 
    & Link Pred. &  \underline{$0.907$}$\pm0.003$ & $\underline{0.896}\pm0.003$& $\underline{0.806}\pm0.009$ & $0.911\pm0.002$\\  
    & Node Class. & \underline{0.550}& \underline{0.758} & \underline{0.201} & \underline{0.140}\\
    & Time &$\approx$1s&  $\approx$5s & $\approx$1s & $\approx$30m\\
    \hline
    \multirow{3}{*}{AUASE}
    & Link Pred. & $\mathbf{0.915}\pm0.004$ & $\underline{0.896}\pm0.005   $&$\mathbf{0.809}\pm0.009$ &$\mathbf{0.954}\pm0.003$\\
    & Node Class. &\textbf{0.668} & \textbf{0.825}& \textbf{0.281}&\textbf{0.173} \\                
    & Time & $\approx$5s & $\approx$5s& $\approx$5s & $\approx$30m \\
    \hline

    \end{tabular}
    % }
\end{table*}

\subsection{Link Prediction}\label{sec:link_pred}

For link prediction, the goal is to predict whether two nodes form an edge in the current network, given historical embeddings. Training data consists of two node embeddings from the same time interval,
\begin{equation*}
    \left( (\hat{\Y}_\A^{(t)})_i, (\hat{\Y}_\A^{(t)})_j \right).
\end{equation*}
These are assigned a positive label if there exists an edge between the two nodes at time $t+1$, namely, $\+A_{ij}^{(t+1)} = 1$, otherwise, they are assigned a negative label, balanced so the classes have the same size. We train a binary classifier using XGBoost \citep{chen2016xgboost} on the first $T-2$ intervals for all the datasets, test on the remaining data and repeat the sampling 10 times.

Table~\ref{tab:static_sc} reports the area under the receiver operating characteristic curve which lies in the interval $[0,1]$ with higher values indicating better prediction. Stability allows using more history in the classifier, which helps AUASE outperform nearly all of its competitors. However, there is still strong information about edge probabilities in unstable embeddings, allowing GloDyNE to perform marginally better for the ACM dataset. There is less distinction between AUASE and UASE which suggests that covariate information is not as beneficial for this task on these data.

\begin{figure}[h]
\centering
  \includegraphics[width=\columnwidth]{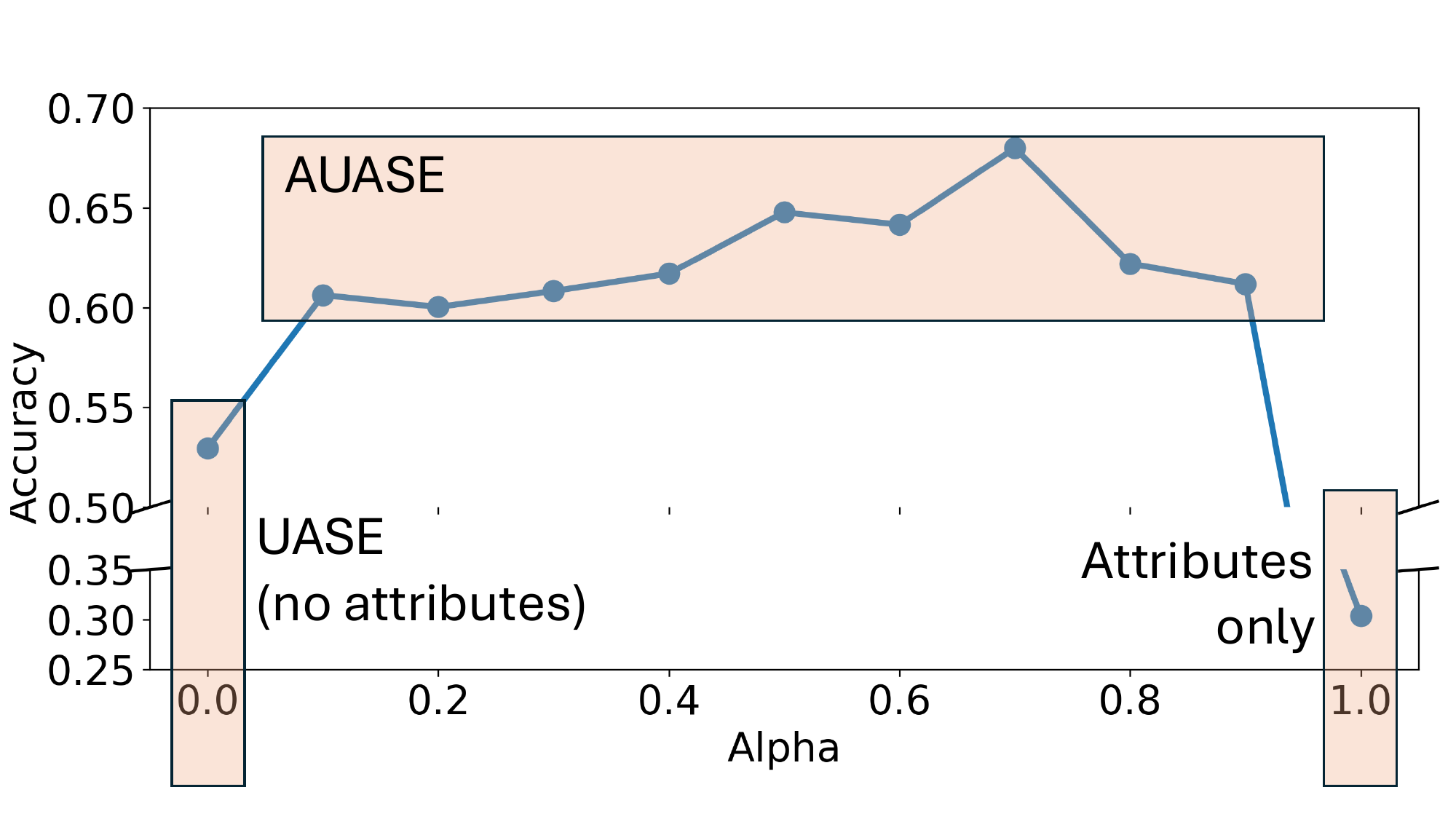}
  
  \caption{AUASE is not sensitive to the attribute weight hyperparameter $\alpha$ as shown by accuracy on DBLP dataset.}\label{fig:sens_alpha}
\end{figure}

\subsection{Parameter Sensitivity} \label{sec:alpha}

In an unsupervised setting, there are many proposed heuristics for hyperparameter estimation, e.g., \cite{tsitsulin2023unsupervised}, which might result in different choices of $\alpha$ in AUASE. Importantly, the \emph{information} contained in the embedding is not sensitive to this choice. To demonstrate this, Figure~\ref{fig:sens_alpha} presents the prediction accuracy for node classification on AUASE embeddings of the DBLP dataset. %across varying $\alpha \in [0,1]$. 

AUASE performs better than competing methods shown in Figure~\ref{fig:dyn_acc}a for any $\alpha \in [0.1,0.9]$ for which attributes and network are jointly embedded, and fine-tuning increases accuracy by 15\% compared to UASE. Since AUASE is computationally efficient, a fast exploration of different hyperparameter choices is possible. An equivalent analysis on link prediction gives similar results (see Appendix~\ref{sec:extra_4}, Table~\ref{tab:alpha_sens}).

\section{Conclusion}
By combining provable stability over time with attributed embedding, we have shown that AUASE learns intrinsically `better' dynamic attributed embeddings that solve problems that all previous (i.e., unstable) methods cannot. This empirically improves predictive performance for node classification and link prediction problems, it is expected to help with unsupervised tasks and aids interpretability.

A limitation of our method is the choice of embedding dimension and of the weight hyperparameter to balance the contribution of network features and attributes. While these are common challenges for unsupervised methods, we have a solution for supervised tasks and demonstrated AUASE robustness to the choice of $\alpha$.

Future research includes exploring non-linear solutions to attributed dynamic embeddings exploiting AUASE stability properties \citep{davis2023simplepowerfulframeworkstable}. Extending AUASE to deal with continuous-time network data \citep{modell2024intensityprofileprojectionframework} could also be valuable.
% \FloatBarrier
\begin{acknowledgements}
Emma Ceccherini gratefully acknowledges support by the Centre for
Doctoral Training in Computational Statistics and Data Science (Compass, EPSRC Grant number
EP/S023569/1). This work was carried out using the computational facilities of the Advanced
Computing Research Centre, University of Bristol - \url{http://www.bris.ac.uk/acrc/}.
Andrew Jones gratefully acknowledges support by the NeST EPSRC programme grant EP/X002195/1. 
\end{acknowledgements}
\clearpage

% References
\bibliography{uai2025-template}

\clearpage

\onecolumn

\title{Supplementary Material}
\maketitle

\appendix

\section{Proof of Theorem \ref{2toinf}}\label{sec:proof1}

Note that row and column permutations of a matrix leave its singular values unchanged, while simply permuting the entries of its left- and right-singular vectors.  Thus, for ease of exposition, we may assume without loss of generality that our matrices $\m{A}_C$ and $\m{P}_C$ take the form
$$\m{A}_C  = \left(\begin{array}{cc}(1-\alpha)\m{A}&\alpha\m{C}\\\alpha\m{C}_*&\m{0}_{p \times Tp}\end{array}\right) \mand \m{P}_C  = \left(\begin{array}{cc}(1-\alpha)\m{P}&\alpha\m{D}\\\alpha\m{D}_*&\m{0}_{p \times Tp}\end{array}\right),$$ 
 where $\m{C} \in \mb{R}^{n \times Tp}$ and $\m{C}_* \in \mb{R}^{p \times Tn}$ are defined by $$\m{C} = \left(\m{C}^{(1)}\,|\,\ldots\,|\,\m{C}^{(T)}\right) \mand \m{C}_* = \left(\m{C}^{(1)\top}\,|\,\ldots\,|\,\m{C}^{(T)\top}\right),$$ and $\m{D}$ and $\m{D}_*$ are the expectation matrices of $\m{C}$ and $\m{C}_*$ respectively.

\begin{proposition}\label{prop8}$\|\m{A}_C - \m{P}_C\| = \OP\left(T^{1/2}\rho^{1/2}n^{1/2}\log^{1/2}(n)\,r_\alpha(1,1/\beta)\right)$.\end{proposition}

\begin{proof} Condition on a choice of latent positions.  A standard application of the triangle inequality then tells us that
$$\|\m{A}_C-\m{P}_C\| \leq (1-\alpha)\|\m{A}-\m{P}\| + \alpha\max\left\{\|\m{C}-\m{D}\|,\|\m{C}_*-\m{D}_*\|\right\}.$$

By Proposition 8 of \cite{jones2021multilayer} we know that $\|\m{A}-\m{P}\| = \OP\left(\rho^{1/2}T^{1/2}n^{1/2}\log^{1/2}(n)\right)$, so it suffices to find a bound for the spectral norm of the centred covariate matrices.

We use the following notation:
\begin{itemize}
\item Let $\m{E}^{(t)}_{il}$ and $\m{F}_i$ denote the $n \times Tp$ (respectively $(n+Tp) \times (n+Tp)$) matrix with $(i,(t-1)p+l)$th (respectively $(i,i)$th) entry equal to $1$ and all other entries equal to $0$.
\item For any matrix $\m{M} \in \mb{R}^{n \times Tp}$, the symmetric dilation of $\m{M}$ is the matrix $\mc{S}(\m{M}) \in \mb{R}^{(n+Tp) \times (n+Tp)}$ given by $$\mc{S}(\m{M}) = \left(\begin{array}{cc}\m{0}_{n \times n}&\m{M}\\\m{M}^\top&\m{0}_{Tp \times Tp}\end{array}\right).$$
\end{itemize}

Note that $$\mc{S}(\m{E}^{(t)}_{il})^k = \left\{\begin{array}{cc}\m{F}_i + \m{F}_{n+(t-1)p+l}&k~\mr{even}\\\mc{S}(\m{E}^{(t)}_{il})&k~\mr{odd}\end{array}\right.$$ for any $i \in [n]$ and $l \in [p]$, and that the matrix $\m{F}_i + \m{F}_{n+(t-1)p+l} - \mc{S}(\m{E}^{(t)}_{il})$ is positive semi-definite, since for any $\m{x} \in \mb{R}^{n+Tp}$ we have $$\m{x}^\top[\m{F}_i+\m{F}_{n+(t-1)p+l}-\mc{S}(\m{E}^{(t)}_{il})]\m{x} = (\m{x}_i - \m{x}_{n+(t-1)p+l})^2 \geq 0.$$

Thus in particular $\mc{S}(\m{E}^{(t)}_{il})^k \preccurlyeq \m{F}_i + \m{F}_{n+(t-1)p+l}$ for all $k$, where $\preccurlyeq$ denotes the standard partial order on positive semi-definite matrices.

Now, note that we can write $$\mc{S}(\m{C} - \m{D}) = \sum_{i=1}^n\sum_{l=1}^p\left[\sum_{t=1}^T\rho^{1/2}(f_l(Z_i^{(t)})-\mb{E}[f_l(Z_i^{(t)})])\mc{S}(\m{E}^{(t)}_{il})\right]$$ where each of the bracketed terms in the sum is an independent matrix.

Let $$\m{M}_{il} = \sum_{t=1}^T f_l(Z_i^{(t)})-\mb{E}[f_l(Z_i^{(t)})])\mc{S}(\m{E}^{(t)}_{il}).$$  By Assumption \ref{exp tail} and comparison with the exponential distribution we see that $$\mb{E}[\m{M}_{il}^k] \preccurlyeq \frac{2k!}{\beta^k}\left[\sum_{t=1}^T \mc{S}(\m{E}^{(t)}_{il})^2\right]$$ for any $k \geq 2$ (where we have used the fact that $\mc{S}(\m{E}^{(s)}_{il})\mc{S}(\m{E}^{(t)}_{il}) = 0$ if $s \neq t$) and thus we may apply a subexponential version of matrix Bernstein (Theorem 6.2 in \cite{tropp2012tails}) to find that
$$\mb{P}\left(\lambda_{\mr{max}}\left(\mc{S}(\m{C}-\m{D})\right) \geq t\right) \leq (n+p)\exp\left(\frac{-t^2/2\rho}{4\sigma^2/\beta^2 + t/\beta\rho^{1/2}}\right)$$ where $$\sigma^2 = \left\|\sum_{i=1}^n\sum_{j=1}^p \left[\sum_{t=1}^T\mc{S}(\m{E}^{(t)}_{ij})^2\right]\right\| = \left\|\sum_{i=1}^n\sum_{l=1}^p \left[T\m{F}_i + \sum_{t=1}^T\m{F}_{n+(t-1)p+l}\right]\right\|.$$

We find that the sum over all $i \in [n]$ and $l \in [p]$ is a diagonal matrix whose first $n$ entries are equal to $Tp$ and remaining $Tp$ entries are equal to $n$, and thus $\sigma^2 = n$, and so we deduce that $$\lambda_{\mr{max}}\left(\mc{S}(\m{C}-\m{D})\right) = \OP\left( \tfrac{1}{\beta}\rho^{1/2}n^{1/2}\log^{1/2}(n)\right).$$

Since the spectral norm of any matrix is equal to the greatest eigenvalue of its symmetric dilation, we therefore conclude that $$\|\m{C}-\m{D}\| = \OP\left(\tfrac{1}{\beta} \rho^{1/2}n^{1/2}\log^{1/2}(n)\right).$$

An identical argument shows that $\|\m{C}_* - \m{D}_*\| = \OP\left(\tfrac{1}{\beta}T^{1/2}\rho^{1/2}n^{1/2}\log^{1/2}(n)\right)$, from which the result follows.

Thus, taking a union bound and integrating over all possible sets of latent positions gives us our desired result.
\end{proof}

\begin{corollary}\label{coro10} The non-zero singular values $\sigma_i(\m{A}_C)$ satisfy $$\sigma_i(\m{A}_C) = \ThetaP\left(T^{1/2} \rho n\right)$$ for each $i \in [d]$. \end{corollary}

\begin{proof} A corollary of Weyl's inequalities (\cite{horn2012analysis}, Corollary 7.3.5) states that $|\sigma_i(\m{A}_C)-\sigma_i(\m{P}_C)| \leq \|\m{A}_C-\m{P}_C\|$ for all $i$, and so in particular $$\sigma_i(\m{P}_C) - \|\m{A}_C-\m{P}_C\|\leq \sigma_i(\m{A}_C) \leq \sigma_i(\m{P}_C) + \|\m{A}_C-\m{P}_C\|$$
and the result follows directly from Assumption \ref{sing values P ass} and Proposition \ref{prop8}.
\end{proof}

\begin{proposition}\label{prop13} The following bounds hold:
\begin{enumerate}[label=\normalfont(\roman*)]
\item $\|\Uu{A}\Uu{A}^\top - \Uu{P}\Uu{P}^\top\|$, $\|\Vv{A}\Vv{A}^\top - \Vv{P}\Vv{P}^\top\| =  \OP\left(\frac{\log^{1/2} (n)\,r_\alpha(1,1/\beta)}{\rho^{1/2} n^{1/2}}\right)$.
\item $\|\Uu{P}^\top(\m{A}_C-\m{P}_C)\Vv{P}\|_{\mr{Frob}} = \OP\left(\log^{1/2}(n)\,r_\alpha(1,\rho^{1/2}/\beta)\right)$.
\item $\|\Uu{P}^\top(\m{A}_C-\m{P}_C)\Vv{A}\|_{\mr{Frob}}$, $\|\Uu{A}^\top(\m{A}_C-\m{P}_C)\Vv{P}\|_{\mr{Frob}} = \OP\left(T^{1/2}\log(n)\,r_\alpha(1,1/\beta)^2\right)$.
\item $\|\Uu{P}^\top\Uu{A} - \Vv{P}^\top\Vv{A}\|_{\mr{Frob}} = \OP\left(\frac{\log(n)\,r_\alpha(1,1/\beta)^2}{\rho n}\right)$.
\end{enumerate}
\end{proposition}

\begin{proof}
\leavevmode
\begin{enumerate}[label=\normalfont(\roman*)]
\item Denoting by $\theta_1, \ldots, \theta_d$ the principal angles between the subspaces spanned by the columns of $\Uu{P}$ and $\Uu{A}$, a variant of the Davis--Kahan theorem (\cite{yu2015daviskahan}, Theorem 4) states that
  $$\max_{i \in [d]}|\sin(\theta_i)| \leq \frac{2d^{\,{1/2}}\left(2\sigma_1(\m{P}_C)+\|\m{A}_C-\m{P}_C\|\right)\|\m{A}_C-\m{P}_C\|}{\sigma_d(\m{P}_C)^2}.$$
By definition, $\theta_i = \cos^{-1}(\sigma_i)$, where the $\sigma_i$ are the singular values of the matrix $\Uu{P}^\top\Uu{A}$. A standard result states that the non-zero eigenvalues of the matrix $\Uu{A}\Uu{A}^\top - \Uu{P}\Uu{P}^\top$ are precisely the $\sin(\theta_i)$ (each occurring twice) and thus the result follows after applying the bounds from Assumption \ref{sing values P ass} and Proposition \ref{prop8}.  An identical argument gives the result for $\|\Vv{A}\Vv{A}^\top - \Vv{P}\Vv{P}^\top\|$.

\item  Condition on a choice of latent positions.  For any $i,j \in [d]$, and for any $t \in [T]$, let $u$ and $v^{(t)}$ denote the $i$th and $j$th columns of $\Uu{P}$ and $\Vv{P}^{(t)}$ respectively, so that
  $$\bigl(\Uu{P}^\top(\m{A}_C-\m{P}_C)\Vv{P}\bigr)_{ij} = (1-\alpha)(\m{T}_1 + \m{T}_2) + \alpha\m{T}_3,$$
  where
  \begin{align*}
    \m{T}_1 &= \sum_{t=1}^T\sum_{k=2}^n\sum_{l=1}^{k-1}(u_kv^{(t)}_l+u_lv^{(t)}_k)\left(\m{A}^{(t)}_{kl}-\m{P}^{(t)}_{kl}\right)\\
    \m{T}_2 &= \sum_{t=1}^T \sum_{k=1}^n u_kv^{(t)}_k\left(\m{A}^{(t)}_{kk}-\m{P}^{(t)}_{kk}\right)\\
    \m{T}_3 &= \sum_{t=1}^T\sum_{k=1}^n\sum_{l=1}^p \left(u_kv^{(t)}_{n+l}+u_{n+l}v^{(t)}_k\right)\left(\m{C}_{kl}^{(t)}-\m{D}_{kl}^{(t)}\right).
  \end{align*}
  
  The term $\m{T}_2$ can be seen to be $\OP(\rho T)$, and so we can disregard it for the purposes of our analysis.  The terms in the remaining two sums are independent zero-mean random variables, where the terms in $\m{T}_1$ are bounded in absolute value by $|u_kv^{(t)}_l+u_lv^{(t)}_k|$ and so we can apply Hoeffding's inequality to find that $|\m{T}_1| = \OP(\log^{1/2}(n))$ (noting that since $u$ and $v$ have norm at most $1$, the sum of the terms $|u_kv^{(t)}_l+u_lv^{(t)}_k|^2$ is at most $2$).  Similarly, one can use the fact that the sum of the terms $|u_kv^{(t)}_{n+l}+u_{n+l}v^{(t)}_k|^2$ is at most $2$ and apply Assumption \ref{exp tail} to find that $|\m{T}_3| = \OP(\tfrac{1}{\beta}\rho^{1/2}\log^{1/2}(n))$, and so deduce that
$$
  \left|(\Uu{P}^\top(\m{A}^{(t)}_C-\m{P}^{(t)}_C)\Vv{P})_{ij}\right| = \OP\left(\log^{1/2}(n)\,r_\alpha(1,\rho^{1/2}/\beta)\right).
$$ The result is then obtained by taking a union bound over all $i$, $j$ and integrating over all possible choices of latent positions.
\item Observe that, since $\Vv{A}^\top\Vv{A} = \m{I}_d$, we may write
 $$
   \Uu{P}^\top(\m{A}_C-\m{P}_C)\Vv{A} = \Uu{P}^\top(\m{A}_C-\m{P}_C)(\Vv{A}\Vv{A}^\top - \Vv{P}\Vv{P}^\top)\Vv{A} + \Uu{P}^\top(\m{A}_C-\m{P}_C)\Vv{P}\Vv{P}^\top\Vv{A}.
 $$

 These terms satisfy 
 $$
   \|\Uu{P}^\top(\m{A}_C-\m{P}_C)(\Vv{A}\Vv{A}^\top - \Vv{P}\Vv{P}^\top)\Vv{A}\|_{\mr{Frob}} = \OP\left(T^{1/2}\log(n)\,r_\alpha(1,1/\beta)^2\right)
 $$ 
 and 
 $$
   \|\Uu{P}^\top(\m{A}-\m{P})\Vv{P}\Vv{P}^\top\Vv{A}\|_{\mr{Frob}} =\OP\left(\log^{1/2}(n)\,r_\alpha(1,\rho^{1/2}/\beta)\right)
 $$
 where we have applied Proposition \ref{prop8}, the results from parts (i) and (ii), and the fact that
 $$
   \|\m{M}\m{N}\|_{\mr{Frob}} \leq \max\bigl\{\|\m{M}\|\|\m{N}\|_{\mr{Frob}},\|\m{N}\|\|\m{M}\|_{\mr{Frob}}\bigr\}
 $$
 for any commensurate matrices $\m{M}$ and $\m{N}$.  The first of these two terms dominates, from which the result follows, and an identical argument bounds the term $\|\Uu{A}^\top(\m{A}-\m{P})\Vv{P}\|_{\mr{Frob}}$.

\item Note that
  $$
    \Sig{P}(\Uu{P}^\top\Uu{A}-\Vv{P}^\top\Vv{A}) + (\Uu{P}^\top\Uu{A}-\Vv{P}^\top\Vv{A})\Sig{A} = \Uu{P}^\top(\m{A}_C-\m{P}_C)\Vv{A} - \Vv{P}^\top(\m{A}_C-\m{P}_C)^\top\Uu{A}
  $$
  and for any $i, j \in [d]$ the $(i,j)$th entry of the left-hand matrix is equal to
  $$\bigr(\sigma_i(\m{P}_C) + \sigma_j(\m{A}_C)\bigl)(\Uu{P}^\top\Uu{A}-\Vv{P}^\top\Vv{A})_{ij}.$$
  Thus
  \begin{align*}
    \left|(\Uu{P}^\top\Uu{A}-\Vv{P}^\top\Vv{A})_{ij}\right|^2 &= \frac{\left|(\Uu{P}^\top(\m{A}_C-\m{P}_C)\Vv{A} - \Vv{P}^\top(\m{A}_C-\m{P}_C)^\top\Uu{A})_{ij}\right|^2}{\left(\sigma_i(\m{P}_C) + \sigma_j(\m{A}_C)\right)^2} \\[1em]
    &\leq \frac{2\left(\left|(\Uu{P}^\top(\m{A}_C-\m{P}_C)\Vv{A})_{ij}\right|^2 + \left|(\Vv{P}^\top(\m{A}_C-\m{P}_C)^\top\Uu{A})_{ij}\right|^2\right)}{\left(\sigma_d(\m{P}_C) + \sigma_d(\m{A}_C)\right)^2}
  \end{align*}
  and so
  $$
     \left\|\Uu{P}^\top\Uu{A}-\Vv{P}^\top\Vv{A}\right\|_{\mr{Frob}} \leq \frac{2\max\left\{\left\|\Uu{P}^\top(\m{A}_C-\m{P}_C)\Vv{A}\right\|_{\mr{Frob}},\left\|\Uu{A}^\top(\m{A}_C-\m{P}_C)\Vv{P}\right\|_{\mr{Frob}}\right\}}{\sigma_d(\m{P}_C) + \sigma_d(\m{A}_C)}
  $$
  by summing over all $i, j \in [d]$ and noting that the Frobenius norm is invariant under matrix transposition.  The result then follows by applying part (iii), Proposition \ref{prop8} and Corollary \ref{coro10}.
\end{enumerate}
\end{proof}

\begin{proposition}\label{prop14} Let $\Uu{P}^\top\Uu{A} + \Vv{P}^\top\Vv{A}$ admit the singular value decomposition 
\begin{align*}
	\Uu{P}^\top\Uu{A} + \Vv{P}^\top\Vv{A} = \Uu{}\Sig{}\Vv{}^\top,
\end{align*} 
and let $\m{W} = \Uu{}\Vv{}^\top$. Then 
$$
	\max\left\{\|\Uu{P}^\top\Uu{A} - \m{W}\|_{\mr{Frob}}, \|\Vv{P}^\top\Vv{A} - \m{W}\|_{\mr{Frob}}\right\} = \displaystyle\OP\left(\frac{\log(n)\,r_\alpha(1,1/\beta)^2}{\rho n}\right).
$$\end{proposition}

\begin{proof} A standard argument shows that $\m{W}$ satisfies
  $$
    \m{W} = \argmin_{\m{Q} \in \mr{O}(d)} \,\|\Uu{P}^\top\Uu{A} - \m{Q}\|_{\mr{Frob}}^2 + \|\Vv{P}^\top\Vv{A} - \m{Q}\|_{\mr{Frob}}.
  $$
  Let $\Uu{P}^\top\Uu{A} = \Uu{*}\Sig{*}\Vv{*}^\top$ be the singular value decomposition of $\Uu{P}^\top\Uu{A}$, and define $\m{W}_* \in \mathrm{O}(d)$ by $\m{W}_* = \Uu{*}\Vv{*}^\top$. Then, denoting by $\sigma_1, \ldots, \sigma_d$ the singular values of $\Uu{P}^\top\Uu{A}$ and defining $\theta_i = \cos^{-1}(\sigma_i)$ as in Proposition \ref{prop13} (i), we see that
$$
  \|\Uu{P}^\top \Uu{A} - \m{W}_*\|_{\mr{Frob}} = \|\Sig{*} - \m{I}_d \|_{\mr{Frob}} \leq d^{\,{1/2}} \sin^2(\theta_d) \leq d^{\,{1/2}} \|\Uu{A}\Uu{A}^\top - \Uu{P}\Uu{P}^\top\|^2
$$
and so 
$$
	\|\Uu{P}^\top \Uu{A} - \m{W}_*\|_{\mr{Frob}} = \OP\left(\frac{\log(n)\,r_\alpha(1,1/\beta)^2}{\rho n}\right).
$$

Also, 
$$
	\|\Vv{P}^\top \Vv{A} - \m{W}_*\|_{\mr{Frob}} \leq \|\Vv{P}^\top \Vv{A} - \Uu{P}^\top\Uu{A}\|_{\mr{Frob}} + \|\Uu{P}^\top\Uu{A} - \m{W}_*\|_{\mr{Frob}}
$$
	and so 
$$
	\|\Vv{P}^\top \Vv{A} - \m{W}_*\|_{\mr{Frob}} = \OP\left(\frac{\log(n)\,r_\alpha(1,1/\beta)^2}{\rho n}\right).
$$ 
by Proposition \ref{prop13}, part (iv).

Since by definition 
$$
	\|\Uu{P}^\top\Uu{A} - \m{W}\|_{\mr{Frob}}^2 + \|\Vv{P}^\top\Vv{A} - \m{W}\|_{\mr{Frob}}^2 \leq \|\Uu{P}^\top\Uu{A} - \m{W}_*\|_{\mr{Frob}}^2 + \|\Vv{P}^\top\Vv{A} - \m{W}_*\|_{\mr{Frob}}^2 
$$
the result follows.
\end{proof}

\begin{proposition}\label{prop15} The following bounds hold:
\begin{enumerate}[label=\normalfont(\roman*)]
	\item $\|\m{W}\Sig{A} - \Sig{P}\m{W}\|_{\mr{Frob}} = \OP\left(T^{1/2}\log(n)\,r_\alpha(1,1/\beta)^2\right)$.
	\item $\|\m{W}\Sig{A}^{1/2} - \Sig{P}^{1/2}\m{W}\|_{\mr{Frob}} = \OP\left(\frac{T^{1/4}\log(n)\,r_\alpha(1,1/\beta)^2}{\rho^{1/2}n^{1/2}}\right)$.
	\item $\|\m{W}\Sig{A}^{-{1/2}} - \Sig{P}^{-{1/2}}\m{W}\|_{\mr{Frob}} =\OP\left(\frac{\log(n)\,r_\alpha(1,1/\beta)^2}{\rho n}\right)$.
\end{enumerate}
\end{proposition}

\begin{proof}
\leavevmode
\begin{enumerate}[label=\normalfont(\roman*)]
	\item Observe that
$$
	\m{W}\Sig{A} - \Sig{P}\m{W} = (\m{W}-\Uu{P}^\top\Uu{A})\Sig{A} + \Uu{P}^\top(\m{A}_C-\m{P}_C)\Vv{A} + \Sig{P}(\Vv{P}^\top\Vv{A} - \m{W})
$$
and so
$$
  \|\m{W}\Sig{A} - \Sig{P}\m{W}\|_{\mr{Frob}} = \OP\left(T^{1/2}\log(n)\,r_\alpha(1,1/\beta)^2\right)
  $$
  by Assumption \ref{sing values P ass}, Corollary \ref{coro10} and Propositions \ref{prop13} and \ref{prop14}.

\item Note that 
$$
	(\m{W}\Sig{A}^{1/2}-\Sig{P}^{1/2} \m{W})_{ij} = \frac{(\m{W}\Sig{A} - \Sig{P}\m{W})_{ij}}{\sigma_j(\m{A})^{1/2} + \sigma_i(\m{P})^{1/2}},
$$
and so the result follows by applying part (i), Assumption \ref{sing values P ass}, Corollary \ref{coro10} and summing over all $i, j \in [d]$.

\item Note that 
$$
  (\m{W}\Sig{A}^{-{1/2}}-\Sig{P}^{-{1/2}} \m{W})_{ij} = \frac{\bigl(\m{W}\Sig{A}^{1/2}-\Sig{P}^{1/2} \m{W}\bigr)_{ij}}{\sigma_i(\m{P})^{1/2}\,\sigma_j(\m{A})^{1/2}}
$$
and so the result follows as in part (ii).
\end{enumerate} 
\end{proof}

\begin{proposition}\label{prop19} Let 
\begin{align*}
    \m{R}_{1,1} &= \Uu{P}(\Uu{P}^\top \Uu{A} \Sig{A}^{1/2} - \Sig{P}^{1/2}\m{W})\\
    \m{R}_{1,2} &= (\m{I}-\Uu{P}\Uu{P}^\top)(\m{A}_C-\m{P}_C)(\Vv{A}- \Vv{P}\m{W})\Sig{A}^{-1/2}\\
    \m{R}_{1,3} &= -\Uu{P}\Uu{P}^\top(\m{A}_C-\m{P}_C)\Vv{P}\m{W}\Sig{A}^{-1/2}\\
    \m{R}_{1,4} &= (\m{A}_C-\m{P}_C) \Vv{P}(\m{W}\Sig{A}^{-1/2} - \Sig{P}^{-1/2}\m{W})
\end{align*}
and
\begin{align*}
    \m{R}_{2,1} &= \Vv{P}(\Vv{P}^\top \Vv{A}\Sig{A}^{1/2} - \Sig{P}^{1/2}\m{W})\\
    \m{R}_{2,2} &= (\m{I}-\Vv{P}\Vv{P}^\top)(\m{A}_C-\m{P}_C) (\Uu{A}- \Uu{P}\m{W})\Sig{A}^{-1/2}\\
    \m{R}_{2,3} &= -\Vv{P}\Vv{P}^\top(\m{A}_C-\m{P}_C)^\top\Uu{P}\m{W}\Sig{A}^{-1/2}\\
    \m{R}_{2,4} &= (\m{A}_C-\m{P}_C)^\top\Uu{P}(\m{W}\Sig{A}^{-1/2} - \Sig{P}^{-1/2}\m{W})
\end{align*}
and let $\hat{\m{R}}_{1,i}$ and $\hat{\m{R}}_{2,i}$ denote the restrictions to the first $n$ and $Tn$ rows of $\m{R}_{1,i}$ and $\m{R}_{2,i}$ respectively.  Then the following bounds hold:

\begin{enumerate}[label=\normalfont(\roman*)]
\item $\| \hat{\m{R}}_{1,1} \|_{2\to \infty} = \OP\left(\frac{T^{1/4}\log(n)\,r_\alpha(1,1/\beta)^2}{\rho^{1/2}n}\right)$ and $\| \hat{\m{R}}_{2,1} \|_{2\to \infty} = \OP\left(\frac{\log(n)\,r_\alpha(1,1/\beta)^2}{T^{1/4}\rho^{1/2}n}\right)$.
\item $\| \hat{\m{R}}_{1,2} \|_{2\to \infty} = \OP\left(\frac{T^{1/4}\log(n)\,r_\alpha(1,1/\beta)^2}{\rho^{1/2} n^{3/4}}\right)$ and $\| \hat{\m{R}}_{2,2} \|_{2\to \infty} = \OP\left(\frac{\log(n)\,r_\alpha(1,1/\beta)^2}{T^{1/4}\rho^{1/2} n^{3/4}}\right)$.
\item $\| \hat{\m{R}}_{1,3} \|_{2\to \infty} = \OP\left(\frac{\log^{1/2}(n)\,r_\alpha(1,1/\beta)}{T^{1/4}\rho^{1/2}n}\right)$ and $\| \hat{\m{R}}_{2,3} \|_{2\to \infty} = \OP\left(\frac{\log^{1/2}(n)\,r_\alpha(1,1/\beta)}{T^{3/4}\rho^{1/2}n}\right)$.
\item $\| \hat{\m{R}}_{1,4} \|_{2\to \infty}, \| \hat{\m{R}}_{2,4} \|_{2\to \infty}= \OP\left(\frac{\log^{3/2}(n)\,r_\alpha(1,1/\beta)^3}{\rho n}\right)$.
\end{enumerate}

\end{proposition}
\begin{proof}
We give full proofs of the bounds only for the terms $\hat{\m{R}}_{1,i}$, noting that any differences in the proofs for the terms $\hat{\m{R}}_{2,i}$ can be derived.  Observe that for all $i, j$ we have $\|\hat{\m{R}}_{i,j}\|_{2 \to \infty} \leq \|\m{R}_{i,j}\|_{2 \to \infty}$.

\begin{enumerate}[label=\normalfont(\roman*)]
\item Note that $\Uu{P} = \m{P}_C\Vv{P}^\top\Sig{P}^{-1}$ and so using the relation $\|\m{A}\m{B}\|_{2\to\infty} \leq \|\m{A}\|_{2 \to \infty}\|\m{B}\|$ and the fact that the spectral norm is invariant under orthonormal transformations we find that $\|\Uu{P}\|_{2\to \infty} \leq \|\m{P}_C\|_{2 \to \infty}\|\Sig{P}^{-1}\| = \mr{O}(n^{-1/2})$ by Assumption \ref{sing values P ass} and the fact that the Euclidean norm of any row of $\m{P}_C$ is $\mr{O}(T^{1/2}\rho n^{1/2})$.  Noting that
  \begin{align*}
    \|\hat{\m{R}}_{1,1}\|_{2 \to \infty} &\leq \|\Uu{P}\|_{2 \to \infty}\|\Uu{P}^\top\Uu{A}\Sig{A}^{1/2}-\Sig{P}^{1/2}\m{W}\|\\
                                   &\leq \|\Uu{P}\|_{2 \to \infty}\left(\|(\Uu{P}^\top\Uu{A}-\m{W})\Sig{A}^{1/2}\|_{\mr{Frob}}+\|\m{W}\Sig{A}^{1/2}-\Sig{P}^{1/2}\m{W}\|_{\mr{Frob}}\right)
  \end{align*}
the result follows by applying Proposition \ref{prop14} and \ref{prop15} and Corollary \ref{coro10}. 

For $\hat{\m{R}}_{2,1}$, we note that the Euclidean norm of any column of $\m{P}_C$ is $\mr{O}(\rho n^{1/2})$, and so $\|\Vv{P}\|_{2\to \infty} = \mr{O}((Tn)^{-1/2})$, and the rest of the proof follows in the same manner as before.

\item We begin by splitting the term $\m{R}_{1,2} = \m{M}_1 + \m{M}_2 + \m{M}_3$, where
  \begin{align*}
    \m{M}_1 &= \Uu{P}\Uu{P}^\top (\m{A}_C-\m{P}_C)(\Vv{A} - \Vv{P}\m{W}) \Sig{A}^{-1/2}\\
    \m{M}_2 &= (\m{A}_C-\m{P}_C)\Vv{P}(\Vv{P}^\top\Vv{A}-\m{W})\Sig{A}^{-1/2}\\ 
    \m{M}_3 &= (\m{A}_C-\m{P}_C)(\m{I}-\Vv{P}\Vv{P}^\top)\Vv{A}\Sig{A}^{-1/2}
  \end{align*}

  Now, $$\| \hat{\m{M}}_1 \|_{2\to \infty} \leq \| \Uu{P}\|_{2\to \infty} \|\m{A}_C-\m{P}_C\| \|\Vv{A} - \Vv{P}\m{W}\| \|\Sig{A}^{-1/2}\|$$
  where 
  \begin{align*}
    \|\Vv{A}-\Vv{P} \m{W}\| & \leq \|\Vv{A}-\Vv{P}\Vv{P}^\top\Vv{A}\| + \|\Vv{P}(\Vv{P}^\top\Vv{A}-\m{W})\|\\
                          &= \|\Vv{A}\Vv{A}^\top-\Vv{P}\Vv{P}^\top\| + \|\Vv{P}^\top\Vv{A}-\m{W}\|\\
                          & = \OP\left(\frac{\log^{1/2}(n)\,r_\alpha(1,1/\beta)}{\rho^{1/2} n^{1/2}}\right)
  \end{align*}
  by Propositions \ref{prop13} and \ref{prop14} and our assumptions regarding the asymptotic growth of $\rho$, and so
 $$\| \hat{\m{M}}_1\|_{2\to \infty} = \OP\left(\frac{T^{1/4} \log(n)\,r_\alpha(1,1/\beta)^2}{\rho^{1/2}n}\right)$$
 by combining this with Proposition \ref{prop8}  and Corollary \ref{coro10}.

 Next, note that
   $$\|\hat{\m{M}}_2\|_{2\to\infty} \leq \|(\m{A}_C-\m{P}_C)\Vv{P}\|_{2\to\infty}\|\Vv{P}^\top\Vv{A}-\m{W}\|\|\Sig{A}^{-1/2}\|.$$
                            
 To bound the term $\|(\m{A}_C-\m{P}_C)\Vv{P}\|_{2 \to \infty}$, one can use an identical argument to that of Proposition \ref{prop13}(ii) to find that the absolute value of each term in the matrix is $\OP(\log^{1/2}(n)\,r_\alpha(1,\rho^{1/2}/\beta))$, and since each row has $d$ elements, the bound also holds for the 2-to-infinity norm.
 By applying Proposition \ref{prop14} and Corollary \ref{coro10} we then find that $$\|\hat{\m{M}}_2\|_{2\to\infty} = \OP\left(\frac{\log^{3/2}(n)\,r_\alpha(1,1/\beta)^3}{T^{1/4}\rho^{3/2} n^{3/2}}\right).$$

To bound $\|\hat{\m{M}}_3\|_{2\to\infty}$, let $\hat{\m{M}} = (\hat{\m{A}}_C-\hat{\m{P}}_C)(\m{I}-\Vv{P}\Vv{P}^\top) \Vv{A}\Vv{A}^\top$, so that $\hat{\m{M}}_3 =\hat{\m{M}}\Vv{A}\Sig{A}^{-1/2}$ and thus
$$\|\hat{\m{M}}_3\|_{2 \to \infty}\leq \|\hat{\m{M}}\|_{2 \to\infty}\|\Vv{A}\Sig{A}^{-1/2}\|.$$

The term $\|\Vv{A}\Sig{A}^{-1/2}\|$ is $\OP\left(\frac{1}{T^{1/4}\rho^{1/2} n^{1/2}}\right)$ by Corollary \ref{coro10}, so it remains to bound $\| \hat{\m{M}}\|_{2 \to\infty}$. 

To do so, we claim that the Frobenius norm of the rows of $\hat{\m{M}}$ are exchangeable and thus have the same expected value, which in turn implies that $\mb{E}[\|\hat{\m{M}}\|_{\mr{Frob}}^2] = n \mb{E}[\|\hat{\m{M}}_i\|^2]$ for any $i \in [n]$. Applying Markov's inequality, we therefore see that 
$$\mb{P}(\|\hat{\m{M}}_i\|>r) \leq \frac{\mb{E}[\|\hat{\m{M}}_i\|^2]}{r^2} =\frac{\mb{E}[\|\hat{\m{M}}\|_{\mr{Frob}}^2]}{nr^2}.$$

Now, 
$$\| \hat{\m{M}}\|_{\mr{Frob}} \leq\| \m{M}\|_{\mr{Frob}}=\OP\left(T^{1/2}\log(n)\,r_\alpha(1,1/\beta)^2\right)$$
by Propositions \ref{prop8} and \ref{prop13} (where we note that $(\m{I}-\Vv{P}\Vv{P}^\top)\Vv{A}\Vv{A}^\top = (\Vv{A}\Vv{A}^\top-\Vv{P}\Vv{P}^\top)\Vv{A}\Vv{A}^\top$ to apply the latter) and so it follows that 
$$\|\hat{\m{M}}\|_{2 \to \infty} = \OP\left(\frac{T^{1/2}\log(n)\,r_\alpha(1,1/\beta)^2}{n^{1/4}}\right)$$
and 
$$\|\hat{\m{M}}_3\|_{2 \to \infty} = \OP\left(\frac{T^{1/4}\log(n)\,r_\alpha(1,1/\beta)^2}{\rho^{1/2} n^{3/4}}\right).$$

To see that our claim is true, let $\m{Q} \in \mr{O}(n)$ be a permutation matrix, and let $$\m{Q}_* = \mr{diag}(\underbrace{\m{Q},\ldots,\m{Q}}_{T},\m{I}_{Tp}) \in \mr{O}(T(n+p)).$$
Since the latent positions $\m{Z}_i$ are assumed to be iid, the matrix entries of the pairs $(\m{A}_C,\m{P}_C)$ and $(\m{Q}\m{A}_C\m{Q}_*^\top,\m{Q}\m{P}_C\m{Q}_*^\top)$ have the same joint distribution, as these transformations simply correspond to a relabelling of the nodes, while the left- and right-singular vectors of $\m{Q}\m{A}_C\m{Q}_*^\top$ (respectively $\m{Q}\m{P}_C\m{Q}_*^\top$) are given by $\m{Q}\Uu{A}$ and $\m{Q}_*\Vv{A}$ (respectively $\m{Q}\Uu{P}$ and $\m{Q}_*\Vv{P}$).

Thus, since
$$\m{Q}\hat{\m{M}}\m{Q}_*^\top = \m{Q}(\hat{\m{A}}_C-\hat{\m{P}}_C)\m{Q}_*^\top(\m{I}-\m{Q}_*\Vv{P}\Vv{P}^\top\m{Q}_*^\top) \m{Q}_*\Vv{A}\Vv{A}^\top\m{Q}_*^\top$$
we deduce that the matrix entries of $\hat{\m{M}}$ and $\m{Q}\hat{\m{M}}\m{Q}_*^\top$ have the same joint distribution, proving our claim.

Combining these results, we see that $$\|\hat{\m{R}}_{1,2}\|_{2\to \infty}= \OP\left(\frac{T^{1/4}\log(n)\,r_\alpha(1,1/\beta)^2}{\rho^{1/2} n^{3/4}}\right)$$
as required. 

The proof for $\hat{\m{R}}_{2,2}$ follows similarly, but involves splitting the matrix $\hat{\m{M}}' = (\hat{\m{A}}_C-\hat{\m{P}}_C)^\top(\m{I}-\Uu{P}\Uu{P}^\top) \Uu{A}\Uu{A}^\top$ into $T$ distinct blocks and applying the same argument to show that the rows of each individual block are exchangeable.  Note that in this case we divide through by a factor of $T^{1/2}$, as the spectral norm of each individual block is $\OP(n^{1/2}\log^{1/2}(n)\,r_\alpha(1,1/\beta))$, and consequently find that $$\|\hat{\m{R}}_{2,2}\|_{2\to \infty}= \OP\left(\frac{\log(n)\,r_\alpha(1,1/\beta)^2}{T^{1/4}\rho^{1/2} n^{3/4}}\right).$$

\item We see that
  \begin{align*}
    \|\hat{\m{R}}_{1,3}\|_{2\to\infty} &\leq \|\Uu{P}\|_{2\to\infty}\|\Uu{P}^\top(\m{A}_C-\m{P}_C)\Vv{P}\|_{\mr{Frob}}\|\m{W}\Sig{A}^{-1/2}\|\\
                           &= \OP\left(\frac{\log^{1/2}(n)\,r_\alpha(1,1/\beta)}{T^{1/4}\rho^{1/2}n}\right)
  \end{align*}
by Proposition \ref{prop13} and Corollary \ref{coro10}.
The proof for $\hat{\m{R}}_{2,3}$ follows identically, noting the additional factor of $T^{1/2}$ in the denominator from $\|\Vv{P}\|_{2 \to \infty}$.
\item We see that
  \begin{align*}
  \|\m{R}_{1,4}\|_{2\to\infty} &\leq \|(\m{A}_C-\m{P}_C)\Vv{P}\|_{2 \to \infty}\|\m{W}\Sig{A}^{-1/2}-\Sig{P}^{-1/2}\m{W}\|_{\mr{Frob}}\\
  &= \OP\left(\frac{\log^{3/2}(n)\,r_\alpha(1,1/\beta)^3}{\rho n}\right)\end{align*}
by applying Proposition \ref{prop15} and bounding the term $\|(\m{A}_C-\m{P}_C)\Vv{P}\|_{2 \to \infty}$ as in part (ii).
\end{enumerate}
\end{proof}

\textbf{Theorem 1}
\begin{proof}
    
We first consider the left embedding $\hat{\m{X}}_{\m{A}}$. Observe that 

$$\hat{\m{X}}_{\m{A}}-\hat{\m{X}}_{\m{P}}\m{W} = (\hat{\m{A}}_C-\hat{\m{P}}_C)\Vv{P}\Sig{P}^{-1/2}\m{W} + \sum_{i=1}^4 \hat{\m{R}}_{1,i}$$

and so 
$$\|\hat{\m{X}}_{\m{A}}-\hat{\m{X}}_{\m{P}}\m{W}\|_{2 \to \infty} \leq \sigma_d({\m{P}})^{-1/2}\|(\hat{\m{A}}_C-\hat{\m{P}}_C)\Vv{P}\|_{2 \to \infty} + \sum_{i=1}^4 \|\hat{\m{R}}_{1,i}\|_{2\to\infty}.$$

As shown in the proof of Proposition \ref{prop19}, $\|(\hat{\m{A}}_C-\hat{\m{P}}_C)\Vv{P}\|_{2\to\infty} = \OP(\log^{1/2}(n)\,r_\alpha(1,1/\beta))$, and consequently we see that
\begin{align*}
  \|\hat{\m{X}}_{\m{A}}-\hat{\m{X}}_{\m{P}}\m{W}\|_{2 \to \infty} &= \OP\left(\frac{\log^{1/2}(n)\,r_\alpha(1,1/\beta)}{T^{1/4}\rho^{1/2}n^{1/2}}\right) + \sum_{i=1}^4 \|\hat{\m{R}}_{1,i}\|_{2\to\infty}\\
                                                                  &= \OP\left(\frac{\log^{1/2}(n)\,r_\alpha(1,1/\beta)}{T^{1/4}\rho^{1/2}n^{1/2}}\right)
\end{align*}
by using the bounds on $\|\hat{\m{R}}_{1,i}\|_{2 \to \infty}$ from Proposition \ref{prop19}.

Similarly, for the embeddings $\hat{\m{Y}}^{(t)}_{\m{A}}$ we find that
$$\hat{\m{Y}}^{(t)}_{\m{A}}-\hat{\m{Y}}^{(t)}_{\m{P}}\m{W} = (\hat{\m{A}}^{(t)}_C-\hat{\m{P}}^{(t)}_C)^\top\Uu{P}\Sig{P}^{-1/2}\m{W} + \sum_{i=1}^4 \hat{\m{R}}_{2,i}$$

and so 
$$\|\hat{\m{Y}}^{(t)}_{\m{A}}-\hat{\m{Y}}^{(t)}_{\m{P}}\m{W}\|_{2 \to \infty} \leq \sigma_d({\m{P}})^{-1/2}\|(\hat{\m{A}}^{(t)}_C-\hat{\m{P}}^{(t)}_C)^\top\Uu{P}\|_{2 \to \infty} + \sum_{i=1}^4 \|\hat{\m{R}}_{1,i}\|_{2\to\infty},$$ and an identical argument yields the same bound.
 \end{proof}

\section{Proof of Lemma \ref{stab-lemma}} \label{sec:prooflemma}

\begin{proof}
For node/time pairs $(i,s)$ and $(j,t)$ to be exchangeable in the attributed dynamic latent position model (Definition~\ref{def:ADCM}), for all $\+Z \in \mathcal{Z} \cup \mathcal{I}$,
\begin{align*}
    H(\Z_i^{(s)}, \Z) &= H(\Z_j^{(t)}, \Z).
\end{align*}
This implies that the corresponding rows of the mean attributed unfolded adjacency matrix are equal, $(\P_C^{(s)})_i = (\P_C^{(t)})_j$. Since $\X_\P (\Y_\P^{(t)})^\top = \P_C^{(t)}$ is a symmetric matrix for all $t \in [T]$,
\begin{align*}
    \Y_\P^{(t)} = \P^{(t)} \X_\P(\X_\P^\top \X_\P)^{-1},
\end{align*}
which implies that the corresponding rows of the noise-free dynamic embedding are equal, $(\Y_\P^{(s)})_i = (\Y_\P^{(t)})_j$.

By Theorem~\ref{2toinf},
\begin{align*}
    \lVert \hat{\Y}_i^{(s)} - (\Y^{(s)}_{\P})_i \W \rVert, \
    \lVert \hat{\Y}_j^{(t)} - (\Y^{(t)}_{\P})_j \W \rVert &= \OP\left(\frac{\log^{1/2}(n)\,r_\alpha(1,\log^\gamma(n))}{T^{1/4}\rho^{1/2}n^{1/2}}\right),
\end{align*}
which, by the triangle inequality and the equality of the noise-free embeddings, shows that
\begin{align*}
    \lVert \hat{\Y}_i^{(s)} - \hat{\Y}_j^{(t)} \rVert &= \OP\left(\frac{\log^{1/2}(n)\,r_\alpha(1,\log^\gamma(n))}{T^{1/4}\rho^{1/2}n^{1/2}}\right). \qedhere
\end{align*}
\end{proof}
\clearpage

\section{Details of the simulated example of Section \ref{Sec:Example}}\label{sec:extra_3}

\subsection{Experimental setup details}

We define the attributed dynamic latent position model in Definition~\ref{def:ADCM},
\begin{align*}
    \+A_{ij}^{(t)} \mid Z_i^{(t)}, Z_j^{(t)} &\stackrel{\text{ind}}{\sim} \t{Bernoulli}(\+B_{Z_i^{(t)}, Z_j^{(t)}}),\\
    \+C_{i\ell}^{(t)} \mid Z_i^{(t)} &\stackrel{\text{ind}}{\sim} \t{Normal}(\+D_{Z^{(t)}_i, \ell}, \sigma^2 \+I),
\end{align*}
with community edge probability matrix $\+B \in \mathbb{R}^{3 \times 3}$ and mean attribute matrix $\+D \in \mathbb{R}^{3 \times p}$,
\begin{equation*}
    \+B = \begin{pmatrix}
    p_1 & p_0 & p_0\\
    p_0 & p_0 & p_0\\
    p_0 & p_0 & p_0
\end{pmatrix}, \quad 
    \+D = \begin{pmatrix}
    \bm{\mu}_1 \\
    \bm{\mu}_1 \\
    \bm{\mu}_2
\end{pmatrix},
\end{equation*}
where $p_0 = 0.5$, $p_1 = 0.3$, $\bm{\mu}_1 = [\underbrace{0, \dots 0}_\text{0-19},\underbrace{1, \dots 1}_\text{20-74},\underbrace{0, \dots 0}_\text{ 75-149}] $ and $\bm{\mu}_2 =  [\underbrace{0, \dots 0}_\text{0-79},\underbrace{1, \dots 1}_\text{80-139},\underbrace{0, \dots 0}_\text{ 140-149}]$.

\subsection{Hyperparameter selection}
We set the embedding dimension $d =3$ for all the methods, the known number of communities. Experiments with other embedding dimension are given in \url{https://anonymous.4open.science/r/AUASE-80E4/README.md}. The selection of hyperparameters for each method shown in Figure~\ref{Figure1} is outlined below:
\begin{itemize}
    \item UASE - no other parameters. 
    \item AUASE - visually chosen parameters (see Figure~\ref{fig:AUASE_alpha}): $\alpha = 0.2$.
    \item DRLAN - visually chosen parameters (see \url{https://anonymous.4open.science/r/AUASE-80E4/README.md}): $\beta = 0.7$, $p = 2$, $q = 1$, $\alpha = [1,0.1,0.001,0.0001]$, $\theta = [1,1,0.1,0.001,0.0001]$. 
    \item DySAT - default parameters. %\texttt{epochs=200, valfreq=1, testfreq=1, batchsize=512, maxgradientnorm=1.0, useresidual='False', negsamplesize=10, walklen=20, negweight=1.0, learningrate=0.001, spatialdrop=0.1, temporaldrop=0.5, weightdecay=0.0005, positionffn='True', window=-1}. 
\end{itemize}

\subsection{Additional figures}\label{sec:extra_plot3}

\begin{figure*}
\centering
  \includegraphics[width=\textwidth]{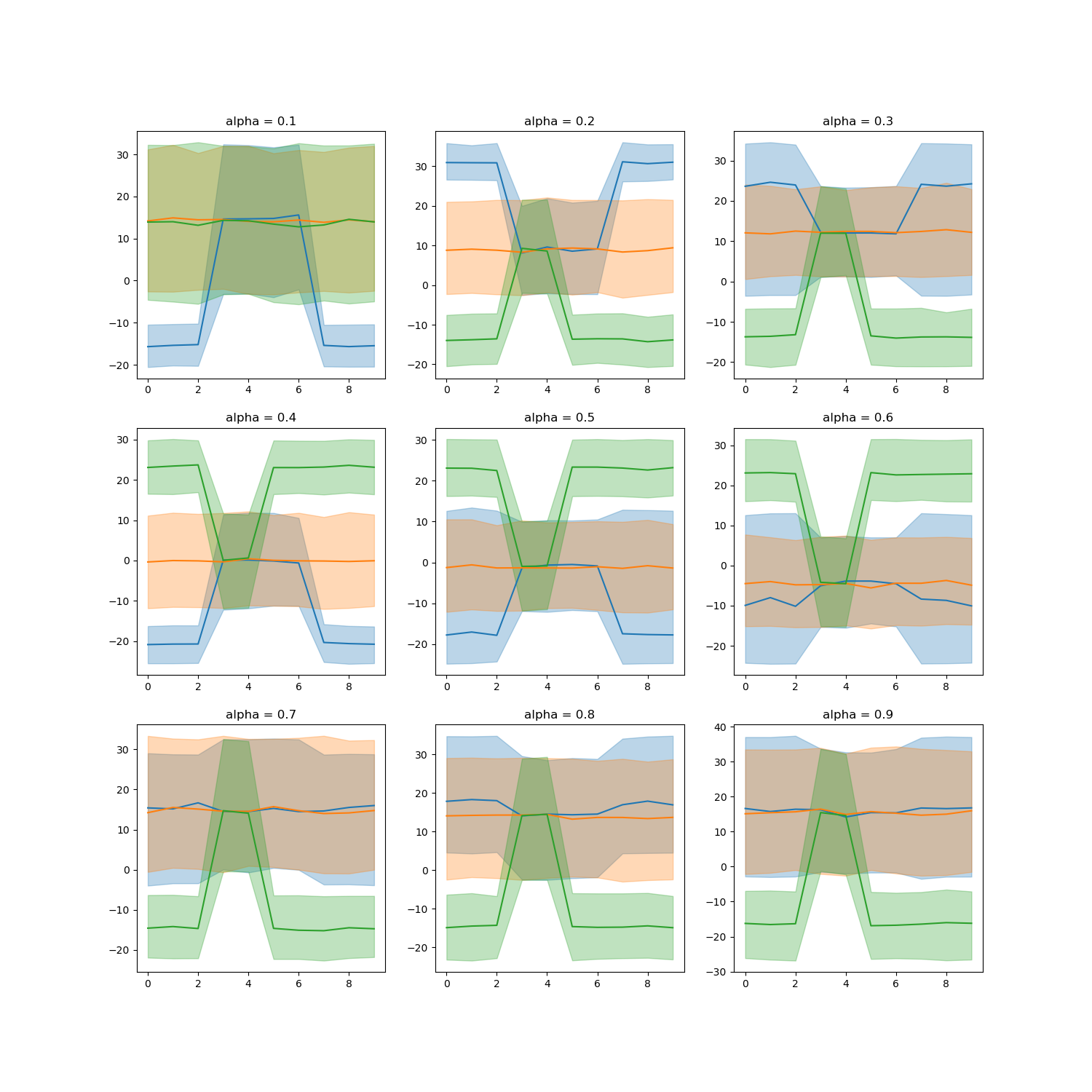}
  \caption{One-dimensional UMAP visualisation of the AUASE node embeddings for varying $\alpha \in[0.1,0.9]$. The coloured lines show the mean embedding for each community with a 90\% confidence interval.}
    \label{fig:AUASE_alpha}
\end{figure*}

\FloatBarrier
\section{Details of the real data analysis of Section \ref{Sec:RealData}} \label{sec:extra_4}

\subsection{Experimental setup details}

%\subsubsection{Data pre-processing}
\textbf{DBLP.} The raw data can be downloaded from \url{https://www.aminer.cn/citation}  as \textit{DBLP-Citation-network V3}. We consider papers published from 1996 to 2009. We group years 1996 - 2000 and 2001 - 2002, while we consider the remaining years as a single time point so that in each time point we have a similar order of magnitude of papers. We only consider authors who published at least three papers between 1996 and 2009. 

To construct the adjacency matrices, we form an edge if two authors collaborated on at least one paper within that time point. For the attributes, we consider words from titles and abstracts of the papers published by an author within a certain time point. We remove stop words, words that have less than three characters, and any symbol that is not an English word. Then, we remove words that have an overall count of less than 100 and the ones that have an overall count of more than 5000. 

We label each paper from the venue where it is published as follows:
\begin{itemize}
    \item Computer Architecture: PPOPP, PPOPP/PPEALS, ASP-DAC, EURO-DAC, DAC, MICRO, PODC, DIALM-PODC.
    \item Computer Network: SIGCOMM, MOBICOM, MOBICOM-CoRoNet, INFOCOM, SenSys. 
    \item Data Mining: SIGMOD Conference, SIGMOD Workshop, Vol. 1, SIGMOD Workshop, Vol. 2, SIGSMALL/SIGMOD Symposium, SIGMOD Record, ICDE, ICDE Workshops, SIGIR, SIGIR Forum.
    \item Computer Theory: STOC, SODA, CAV, FOCS.
    \item Multi-Media: SIGGRAPH, IEEE Visualization, ICASSP.
    \item Artificial Intelligence: IJCAI, IJCAI (1), IJCAI (2), ACL2, ACL, NIPS.
    \item Computer-Human Interaction: IUI, PerCom, HCI.
\end{itemize}

Then, each author is assigned a label at each time point based on the majority label of their papers published within that time point.

\textbf{ACM.} The raw data can be downloaded from \url{https://www.aminer.cn/citation}  as \textit{ACM-Citation-network V8}. We consider the years from 2000 to 2014; each year is a time point. The network and the attributes are constructed in the same manner as for DBLP. The labels are constructed as follows:
\begin{itemize}
    \item Data Science: VLDB, SIGMOD, PODS, ICDE, EDBT, SIGKDD, ICDM, 
          DASFAA, SSDBM, CIKM, PAKDD, PKDD, SDM, DEXA.
    \item Computer Vision: CVPR, ICCV, ICIP, ICPR, ECCV, ICME, ACM-MM.
    
\end{itemize}

\textbf{Epinions.} The raw data can be downloaded from \url{https://www.cse.msu.edu/~tangjili/datasetcode/truststudy.htm}. We consider the years from 2001 to 2011, and we consider users who establish at least two trust relationships in this time frame. Each year is a time point, and an edge is formed if the users trusted each other within that year. 

The attributes are the words of the titles of the reviews published by each author within that year. We remove stop words, words with less than three characters, and words with an overall count of less than 100 or more than 10000.  

To construct the labels, we consider the most frequent category of the reviews published by an author in a given year. We merge some sparse categories with similar, more popular ones; the resulting labels are \textit{Books, Business \& Technology, Cars \& Motorsports, Computers \& Internet, Education, Electronics, Games, Home and Garden, Hotels \& Travel, Kids \& Family, Magazines \& Newspapers, Movies, Music, Musical Instruments, Online Stores \& Services, Personal Finance, Pets, Restaurants \& Gourmet, Software, Sports \& Outdoors,  Wellness \& Beauty}.

\textbf{ogbn-mag.} The raw data can be downloaded from \url{https://ogb.stanford.edu/docs/nodeprop/}. The data spans years from 2010 to 2019, each year is a time point. We consider authors which published at least ten paper in the whole time frame. 

To construct the adjacency matrices, we form an edge if two authors collaborated on at least one paper within that time point.  Each paper is associated with a 128-dimensional word2vec feature vector, to construct the attributes matrices for each author we average the feature vectors of the papers they published in that time point. An author label is set to the label associated with the majority of the papers they published in that time point. 

We use ogbn-mag data to construct a datasets which is suited to our dynamic tasks. The data format and the tasks are at author-level, different than the usual task performed on obgn-mag (`the task is to predict the venue of each paper, given its content, references, authors, and authors’ affiliations'). That is because we are interested in dynamic tasks; authors are dynamic nodes, unlike papers, and their attributes and labels change over time. This creates a much harder task, hence our performance cannot be directly compared to other performances on obgn-mag to predict papers labels.

CONN, Glodyne, DySAT and DyRep are not computationally efficient enough to compute the embeddings of ogbn-mag. We report here the computational barriers for each methods. 

\begin{itemize}
    \item CONN: the authors' code is designed with dense matrices, which would require about 2T of memory for ogbn-mag matrices. 
    \item Glodyne: we capped computational time after 48h.
    \item Dyrep: it requires an unreasonably large amount of memory. The largest dataset the authors' analyses has only $\approx10 000$ nodes. 
\end{itemize}

For all the data sets, the nodes that are not active within a certain time point are considered \textit{Unlabelled}. These nodes are not used for training or testing. 

\subsection{Method comparison}\label{sec:meth_selec}
The following list shows unsupervised dynamic attributes methods we did not include in our method comparison and the reason why. 
\begin{itemize}
    \item \citep{li2017attributed} - source code not available.
    \item \citep{xu2020embedding} - source code not available.
    \item \citep{luodi2024learning} - source code not available.
    \item \citep{liu2021motif} - the code to construct the motif matrices is unavailable, and there is no clear documentation on how to construct them for new datasets. 
    \item \citep{tang2022dynamic} - the source code is undocumented and raises memory errors on the datasets considered in this paper.   \item \citep{mo2024deep} - source code not available.   
    \item \citep{wei2019lifelong} - source code not available.   
    \item \citep{xu2020inductive} - First, the code does not allow for time-varying covariates. Second, it is extremely time-consuming for data sets with large $p$; for example, one epoch on one batch for DBLP takes more than 10 min, meaning the whole computation would be more than 1000 hours. The authors presented examples with a maximum $p =200$, while for the data sets considered in the paper, $p$ is at least 4000. 
    
    \item \citep{ahmed2024learning} - source code not available. 
\end{itemize}

\subsection{Hyperparameter selection}
We choose the embedding dimension using ScreeNOT giving $d = 12$, $d = 29$ and $d = 22$ for DBLP, ACM and Epinions, respectively (for
DySAT, we set $d = 32$ for ACM). We use cross-validation to select $\alpha$ for AUASE, $\beta$ for DRLAN
and $\alpha$ for CONN.
\begin{itemize}
    \item UASE - no other parameters. 
    \item AUASE - no other parameters.
    \item DRLAN - we set all the parameters to the default set by the authors for node classification: $p = 2$, $q = 2$, $\alpha = [1,1,0.1,0.1,1]$, $\theta = [1,10,100]$. 
    \item DySAT - we set all the parameters to the default set by the authors: epochs=200, valfreq=1, testfreq=1, batchsize=512, maxgradientnorm=1.0, useresidual='False',negsamplesize=10, walklen=20, negweight=1.0, learningrate=0.001, spatialdrop=0.1, temporaldrop=0.5, weightdecay=0.0005, positionffn='True', window=-1. 
    \item CONN - we set all the parameters to the default set by the authors: nlayer=2, epochs=200, activate="relu", batchsize=1024, lr=0.01, weightdecay=0.0, hid1=512, hid2=64, losstype=entropy, patience=50, dropout=0.6, drop=1. Note that we implemented CONN in the unsupervised version - \textit{main\_lp.py}.
    \item GloDyNE - we set all the parameters to the default set by the authors: limit=0.1, numwalks=10, walklength=80, window=10, negative=5, workers=32, scheme=4.
    \item DyRep - we set epochs to 20, batch size to 20, and we assume all of the links in Jodie as communication; we set all the other parameters to the default set by the authors. We implemented Dyrep without node features as it only supports non-time-varying attributes, and all three datasets we consider have time-varying attributes. 
\end{itemize}

Note that for DySAT, we set $d = 32$ for ACM instead of $d=29$. DySAT set the embedding dimension \textit{structural\_layer\_config = $d$} as function of the hyperparameter \textit{structural\_head\_config = $d_1$, $d_2$, $d_3$} such that $d_1 \times d_2 =d$. For DBLP we set \textit{structural\_head\_config = 4, 3, 2} giving $d =12$, for Epinions we set \textit{structural\_head\_config = 11, 2, 2} giving $d =22$. However, the only option giving $d =29$ is $d_1 =29$, $d_2 =1$. \textit{structural\_head\_config} represent the dimensions of the layers, and two subsequent layers with such a wide gap would give DySAT an unfair disadvantage. Hence we set \textit{structural\_head\_config = 8,4,4} and \textit{structural\_layer\_config = 32} for ACM. Table~\ref{tab:sent_d} shows that our results are robust to the choice of embedding dimension. 

\subsection{Additional tables}
\begin{table}[h]
    \centering
        \caption{AUC of link prediction on AUASE embeddings for varying $\alpha$. For $\alpha=1$ the method is using only the attributes and not the network, contrarily to AUASE for $\alpha \in (0,1).$}\label{tab:alpha_sens}
    \begin{tabular}{c|c|c|c|c|c|c|c|c|c|c}
    \hline
    $\alpha$ & 0.1 & 0.2 & 0.3&0.4&0.5&0.6&0.7&0.8&0.9&1.0\\
    AUC & 0.915 &0.912& 0.907 &0.898 &0.892 &0.883& 0.887 &0.881 &0.866& 0.546\\
    \hline
    \end{tabular}
\end{table}

\begin{table*}[h]
    \centering
    \begin{tabular*}{0.6\textwidth}{c|c|c|c|c|c|c}
    \hline
    Method & Metrics & 2010 & 2011 & 2012 & 2013 & 2014 \\
    \hline

    \multirow{3}{*}{CONN} & F1  & 0.529& 0.541& 0.537 &0.528& 0.37  \\
    & F1 micro    & 0.535 &0.569 &0.553 &0.542& 0.431 \\
&F1 macro &0.490 &0.496& 0.488& 0.515 &0.394 \\

    \hline

    \multirow{3}{*}{GloDyNE}& F1 &{0.703} &{0.712}  &{0.717} &{ 0.731}  &{0.664} \\ & F1 micro &0.584& 0.589& 0.622& 0.559 &0.486\\
    & F1 macro &0.509 &0.515& 0.535 &0.481& 0.441\\

    \hline

    \multirow{3}{*}{DRLAN} & F1 & 0.483 &0.495 & 0.527 & 0.479 & 0.39 \\ & F1 micro   & 0.627 &  0.599  & 0.619  & 0.556  & 0.466  \\
& F1 macro &0.385& 0.417& 0.450 &0.424& 0.403    \\

    \hline

\multirow{3}{*}{DySAT} & F1 &0.662 & 0.557& 0.583 &0.397&NA\\ & F1 micro &0.680  &0.551& 0.578 &0.556&NA \\
    & F1 macro&0.626& 0.548& 0.558& 0.357 &NA\\

    \hline

    \multirow{3}{*}{UASE}& F1 &   {0.755} & {0.728} & {0.753} & {0.804}  & {0.68} \\ & F1 micro  & 0.772& 0.746 &0.771& 0.81&  0.688 \\
    & F1 macro&0.727& 0.702& 0.722 &0.798 &0.685 \\

    \hline

\multirow{3}{*}{AUASE}
    & F1 & {0.809} & {0.797} & {0.790} & {0.859} & {0.878} \\ & F1 micro  &{0.806}& {0.794}& {0.787}& {0.858}& {0.877} \\
& F1 macro&{0.799} &{0.790} &{0.779} &{0.858}& {0.875}  \\

    \hline
    \end{tabular*}
    \caption{Micro, macro and weighted F1 of node classification on ACM.}
    \label{tab:acm_results}
\end{table*}

\begin{table*}[h]
    \centering
    \begin{tabular*}{0.6\textwidth}{c|c|c|c|c|c|c}
    \hline
    Method & Metrics & 2007 & 2008 & 2009 & 2010 & 2011  \\
    \hline
    \multirow{3}{*}{CONN} 
&F1 & 0.062& 0.081& 0.070&  0.067& 0.067\\
&F1 micro &0.114& 0.133& 0.137& 0.129 &0.131\\
& F1 macro &0.030&  0.038& 0.034& 0.032& 0.032\\

    \hline

    \multirow{3}{*}{GloDyNE} 
    & F1 & 0.194 &0.176 &0.141 &0.140  &0.154\\
    &F1 micro &0.211& 0.189 &0.156& 0.155& 0.165\\
    & F1 macro &0.113 &0.103& 0.087 &0.086 &0.096\\
    \hline

    \multirow{3}{*}{DRLAN} 
    & F1 & 0.092 &0.076 &0.106& 0.067 &0.083\\
    &F1 micro   & 0.119 &0.086& 0.135& 0.098& 0.111 \\
& F1 macro &0.049& 0.039& 0.052& 0.032& 0.044   \\
    \hline
    
\multirow{3}{*}{DyRep}
 & F1 &0.118& 0.107 &0.105& 0.095 &0.101 \\
 &F1 micro &0.15&  0.146 &0.142 &0.132 &0.147 \\
 & F1 macro & 0.074 &0.065& 0.065& 0.054 &0.057\\
\hline

\multirow{3}{*}{DySAT} 
& F1 & 0.062& 0.039 &0.036 &0.011& NA\\
& F1 micro &0.069 &0.095 &0.075&0.042&NA \\
    & F1 macro&0.035& 0.019& 0.021 &0.009  &NA\\
    \hline
    \multirow{3}{*}{UASE} 
    &F1 & 0.224 &0.204 &0.203 &0.196 &0.178\\
    
    &F1 micro  & 0.224 &0.204 &0.203 &0.196 &0.178 \\
    & F1 macro&0.112& 0.085& 0.090 & 0.101& 0.076 \\
    \hline   

    \multirow{3}{*}{AUASE} 
    & F1 & {0.318} &{0.304} &{0.279} &{0.266}& {0.263}\\
    &F1 micro  &{0.332} &{0.319} &{0.294}& {0.279}& {0.277}\\
& F1 macro&{0.197}& {0.176} &{0.167} &{0.163} &{0.154} \\
    \hline
    \end{tabular*}
    \caption{Micro, macro and weighted F1 of node classification on Epinions.}
    \label{tab:epin_results}
\end{table*}

\begin{table}
    \centering
    \begin{tabular*}{0.45\textwidth}{c|c|c|c|c}
    \hline
    Method & Metrics & 2007 & 2008 & 2009 \\
    \hline
    \multirow{3}{*}{CONN}

    & F1 & 0.242& 0.189& 0.141\\
    & F1 micro &0.259& 0.219& 0.181\\
    & F1 macro &0.137&  0.127& 0.110 \\
    \hline

        \multirow{3}{*}{GloDyNE} 
& F1  & 0.138  & 0.252 & 0.077\\
    & F1 micro&0.178& 0.270 & 0.106\\
    & F1 macro&0.117& 0.125& 0.064\\
    \hline
    \multirow{3}{*}{DyRep} 
    & F1 &0.193 &0.319 &0.110 \\
    & F1 micro&0.258 &0.367& 0.190\\
    & F1 macro &0.192 &0.194& 0.091\\
    \hline
    \multirow{3}{*}{DRLAN} 
    & F1 &0.135& 0.241& 0.150 \\
    & F1 micro& 0.193& 0.288& 0.214\\
    & F1 macro &0.115& 0.163& 0.148\\
    \hline
  \multirow{3}{*}{DyRep}
  & F1  &0.193 &0.319 &0.110 \\
& F1 micro&0.258 &0.367& 0.190 \\
& F1 macro &0.192& 0.194 &0.091\\
\hline
    \multirow{3}{*}{DySAT} 
    & F1 & 0.029 &0.267& NA \\
        & F1 micro&0.051 &0.299 &NA\\
    & F1 macro&0.055 &0.133&NA\\
    \hline
    \multirow{3}{*}{UASE}  & F1        & {0.433}  & {0.707} & {0.376} \\
        & F1 micro &0.488&  0.717& 0.446\\
    & F1 macro & 0.460 & 0.525& 0.424\\
    \hline
    \multirow{3}{*}{AUASE} 
    & F1 & 0.584 & 0.757 & 0.564 \\
    & F1 micro & 0.590 & 0.768& 0.584\\
    & F1 macro &0.545& 0.554& 0.543\\
    \hline
    \end{tabular*}
    \caption{Micro, macro and weighted F1 of node classification on DBLP.}
    \label{tab:dblp_results}
\end{table}

\begin{table}[h]
\centering
\begin{tabular*}{0.9\textwidth}{c|c|c|c|c|c|c|c|c|c}
\hline
\multirow{2}{*}{Method} & \multicolumn{3}{c|}{DBLP} & \multicolumn{3}{c|}{ACM} & \multicolumn{3}{c}{Epinions} \\ \cline{2-10} 
                  & $d=32$  & $d=64$  & $d=128$ & $d=32$  & $d=64$  & $d=128$ & $d=32$   & $d=64$   & $d=128$  \\ \hline
CONN &0.532& 0.589 &0.597 & 0.766 &0.779 &0.782 &0.165& 0.176& 0.183\\ \hline
GloDyNE & 0.383 &0.386& 0.394 &0.657& 0.661 &0.662 &0.285& 0.288&0.269\\ \hline
DRLAN & 0.657 & 0.692 &  0.723 &0.785& 0.801 &0.830&0.242& 0.258 &0.277\\ \hline
DyRep & 0.274 & - & -& -& -& -& -& -& - \\
\hline
DySAT& 0.752 &0.780  &0.812 & - & - & - & - & -& - \\ \hline
UASE &0.763& 0.774& 0.781 &0.844 &0.849& 0.852&0.229 &0.228& 0.234\\ \hline
AUASE &0.853& 0.864 &0.872  &0.927 & 0.936 & 0.940 &0.306& 0.335& 0.349\\ \hline
\end{tabular*}
\caption{Accuracies of node classification on DBLP, ACM and Epinions for  $d=32$, $d=64$, and $d=128$. For DySAT and DyRep, we only report partial results due to the intensive computational times}
\label{tab:sent_d}
\end{table}

\end{document}